\documentclass[twoside]{article}

\usepackage[preprint]{aistats2027}
\usepackage[round]{natbib}

\usepackage[utf8]{inputenc} 
\usepackage[T1]{fontenc}    
\usepackage{hyperref}       
\usepackage{url}            
\usepackage{booktabs}       
\usepackage{amsfonts}       
\usepackage{nicefrac}       
\usepackage{microtype}      
\usepackage{xcolor}         

\usepackage{natbib}
\usepackage{caption}
\usepackage{graphicx}
\usepackage{amsmath}
\usepackage{amsthm}
\usepackage{subcaption}
\usepackage{amssymb}
\usepackage{mathtools}
\usepackage{bm}
\usepackage{xspace}
\usepackage[inline]{enumitem}
\usepackage{threeparttable}
\usepackage{tikz}
\usepackage{multirow}
\usepackage{multicol}
\usepackage{comment}
\usepackage{tabularx}
\usepackage{float}

\usepackage{graphicx,bm,colonequals,amsmath,amssymb,url,xcolor,bbm}
\usepackage{array,tabularx,multirow}
\usepackage{soul} 
\usepackage{placeins} 
\usepackage[normalem]{ulem}

\graphicspath{{plots/}}

\newcommand{\todo}[1]{\textcolor{red}{[#1]}}

\usepackage{mathtools}
\mathtoolsset{showonlyrefs} 

\usepackage{amsthm}
\newtheorem{theorem}{Theorem}
\newtheorem{proposition}{Proposition}

\newtheorem{corollary}{Corollary}

\usepackage[ruled,vlined]{algorithm2e}
\SetKwInput{KwInput}{Input}
\SetKwInput{KwOutput}{Output}
\usepackage{algorithmic}
\usepackage{framed}
\usepackage{float}

\newcommand{\bh}{\mathbf{h}}
\newcommand{\bu}{\mathbf{u}}

\newcommand{\bv}{\mathbf{v}}

\newcommand{\bt}{\mathbf{t}}
\newcommand{\bs}{\mathbf{s}}

\newcommand{\bx}{\mathbf{x}}
\newcommand{\by}{\mathbf{y}}
\newcommand{\bq}{\mathbf{q}}
\newcommand{\bg}{\mathbf{g}}
\newcommand{\bk}{\mathbf{k}}

\newcommand{\bw}{\mathbf{w}}
\newcommand{\bo}{\mathbf{o}}
\newcommand{\be}{\mathbf{e}}

\newcommand{\bz}{\mathbf{z}}

\newcommand{\bP}{\mathbf{P}}
\newcommand{\bA}{\mathbf{A}}

\newcommand{\bL}{\mathbf{L}}

\newcommand{\bI}{\mathbf{I}}
\newcommand{\bD}{\mathbf{D}}
\newcommand{\bH}{\mathbf{H}}

\newcommand{\bK}{\mathbf{K}}

\newcommand{\bC}{\mathbf{C}}
\newcommand{\bM}{\mathbf{M}}

\newcommand{\bfzero}{\mathbf{0}}
\newcommand{\bfalpha}{\bm{\alpha}}
\newcommand{\bfgamma}{\bm{\gamma}}
\newcommand{\bfmu}{\bm{\mu}}
\newcommand{\bfxi}{\bm{\xi}}
\newcommand{\bftheta}{\bm{\theta}}

\newcommand{\bfbeta}{\bm{\beta}}
\newcommand{\bfepsilon}{\bm{\epsilon}}

\newcommand{\bfomega}{\bm{\omega}}

\newcommand{\bfphi}{\bm{\phi}}

\newcommand{\bfSigma}{\bm{\Sigma}}

\newcommand{\bfLambda}{\bm{\Lambda}}

\newcommand{\var}{\textbf{Var}}
\newcommand{\cov}{\textbf{Cov}}

\newcommand{\diag}{diag}

\newcommand{\GP}{\mathcal{GP}}

\newcommand{\order}{\mathcal{O}}
\newcommand{\normal}{\mathcal{N}}

\DeclareMathOperator*{\bbE}{\mathbb{E}}
\newcommand{\bbR}{\mathbb{R}}

\newcommand{\domain}{\mathcal{D}}

\newcommand{\lite}[0]{\textsc{LITE}\xspace}

\begin{document}

%

%

\twocolumn[

\aistatstitle{Derivative Gaussian Processes on a Two-Direction Budget }

\aistatsauthor{ Hyunseok Seung \And Matthias Katzfuss }

\aistatsaddress{ University of Wisconsin--Madison \And  University of Wisconsin--Madison } ]

\begin{abstract}
Gradient observations promise more accurate Gaussian process (GP) surrogates, but the cost of incorporating them has long stood in the way of realizing that promise.
We propose a derivative GP with a budget of just two directions per observed gradient.
One direction focuses on each gradient's direct contribution to target prediction, while the other aggregates its indirect contributions through correlations with the conditioning function values.
Within a Vecchia approximation, where each prediction conditions on $m$ nearby inputs in $d$ dimensions, this construction represents their $md$ gradient coordinates using at most $2m$ directional derivatives, giving $\mathcal{O}(m^3)$ dense factorization cost per prediction target.
For general conditioning sets, we bound the posterior approximation error relative to using full gradients and characterize when the error is small or the approximation is exact.
In simulations, our method matches the accuracy of a leading exact gradient-reduction method at equal conditioning set size.
Because its cost grows much more slowly with that size, it can use conditioning sets well beyond the memory limit of the exact method, reaching lower prediction error with a small fraction of the time and memory. 
Notably, our method can exploit gradient observations while requiring less computation time or memory than function-only GP baselines.
Code is available at \href{https://github.com/hseung88/lite}{https://github.com/hseung88/lite}.

\end{abstract}

\section{Introduction}
\label{sec:introduction}
Gradient observations can improve the predictive accuracy of Gaussian process (GP) surrogates~\citep{Solak2002DerivativeOI,Rasmussen2005GaussianPF}, but incorporating full gradients substantially increases the computational and memory costs of inference. Existing scalable derivative GP methods address this problem through kernel interpolation~\citep{Eriksson2018ScalingGP,huang2026scaling}, inducing derivative variables~\citep{Padidar2021ScalingGP}, random feature approximations~\citep{Hung2021ARF}, and structured covariance decompositions~\citep{Roos2021HighDimensionalGP}. Within a Vecchia approximation~\citep{Vecchia1988EstimationAM,Stein2004,Datta2016,Katzfuss2021AGF}, TERA~\citep{Seung2026ScalableDG} recently showed that the $md$ gradient coordinates of $m$ conditioning inputs in $d$ dimensions can be replaced by $m^2$ directional derivatives selected for each prediction target, without changing the target's conditional distribution under a stationary kernel and its gradient noise model. This removes $d$ from the dense linear algebra, but the reduced covariance matrix still requires $\order(m^6)$ factorization cost and $\order(m^4)$ memory per target. Since the accuracy of a Vecchia approximation improves with the size of its conditioning sets~\citep{Katzfuss2021AGF,Guinness2018PermutationAG}, these costs cap the attainable accuracy, and the memory requirement limits how many targets can be processed simultaneously on an accelerator.

A fixed budget of directions per observed gradient would 
alleviate both limitations, provided the selected directions carry the information needed for accurate prediction. In this paper, we propose \lite for \underline{l}ightweight derivative GP \underline{i}nference with \underline{t}wo directional d\underline{e}rivatives per observed gradient. The covariance between each gradient and the target identifies a direction toward the target input; conditioning on the function values introduces a correction that aggregates the directions toward the remaining conditioning inputs. These two components of the conditional covariance determine our two directions, which therefore preserve each gradient's conditional covariance with the target. This gives at most $2m$ directional derivatives per Vecchia conditional, $\order(m^3)$ factorization cost, and $\order(m^2)$ covariance storage. For example, 
with $m=150$, TERA's $22{,}650\times22{,}650$ matrix requires about 4 GB per target, whereas \lite's $450\times450$ matrix requires under 2 MB in double precision, so thousands of targets can be factorized in one batch on a single GPU. We construct the reduced covariance matrix directly from the selected directions.

The two directions preserve each gradient's conditional covariance with the target, but omitted derivatives can still inform the target through their correlations with the retained ones, so the target posterior is not automatically preserved. We show that the reduction is exact when the conditioning geometry is compound-symmetric, the limiting case of distance concentration in high dimensions, and we express the expected KL divergence from the posterior using full gradients in the same conditioning set exactly through interpretable variance ratios, with a two-sided bound and a necessary and sufficient condition for exact recovery.

We evaluate predictive accuracy in controlled GP simulations and on 
a molecular dynamics dataset. We also evaluate optimization performance in Bayesian optimization (BO) tasks, including a real-world hyperparameter optimization benchmark. We compare computation time and memory usage with those of derivative and function-only GP baselines. 
Across regression benchmarks, \lite achieves the lowest runtime
and memory usage among all tested baselines. In high-dimensional
BO tasks with up to $800$ dimensions, it outperforms leading function-only methods while requiring less runtime or memory.

Our contributions are summarized below.
\begin{itemize}[leftmargin=*,itemsep=1pt,topsep=2pt]
\item We derive two directions per observed gradient from its conditional covariance with the prediction target, giving a Vecchia derivative GP with $\order(m^3)$ factorization cost and $\order(m^2)$ covariance memory per target, which enables batched factorization across thousands of targets.
\item We prove that the two directions are exact under compound symmetry and, for general conditioning sets, express the expected KL divergence from the posterior using full gradients exactly, with a two-sided bound and a necessary and sufficient condition for exact recovery.
\item 
In simulations, \lite matches TERA's accuracy at equal $m$ and reaches lower error with conditioning sets beyond TERA's GPU memory limit, using a small fraction of the time and memory.
In GP regression and BO, it also outperforms function-only GP baselines with runtime or memory savings.
\end{itemize}

\section{Related Work}
\label{sec:related}
DDSVGP~\citep{Padidar2021ScalingGP} learns directions at inducing points to construct a sparse variational GP, with the number of directions chosen independently of the input dimension. TERA~\citep{Seung2026ScalableDG} selects directional derivatives of observed gradients for each prediction target, preserving its conditional distribution under the assumed kernel and gradient noise model. In contrast to DDSVGP, our two directions per observed gradient are available in closed form for each target and introduce no additional parameters to optimize. 
Relative to standard derivative GP inference, approximation error arises from restricting conditioning sets through Vecchia and from replacing full gradients with the selected directional derivatives within each set. Relative to TERA, only the second source remains, and Theorem~\ref{thm:methods_approximation_error} characterizes it.

Another approach reduces computational cost by exploiting
structure in derivative covariance matrices or approximating them.
D-SKI~\citep{Eriksson2018ScalingGP} differentiates the approximate
kernel from structured kernel interpolation, enabling fast
matrix-vector products for iterative solvers.
DSoftKI~\citep{huang2026scaling} approximates the kernel through
softmax interpolation and differentiates the interpolation weights
to obtain covariance blocks for full gradient observations.
\citet{Roos2021HighDimensionalGP} exploit algebraic structure in
derivative covariance matrices for exact inference with cost
linear in the input dimension, particularly when the number
of observation locations is much smaller than the input dimension.

\section{Preliminaries}
\label{sec:prelim}
\subsection{Derivative Gaussian Processes}
\label{sec:prelim_dgp}

Let $f:\domain\to\bbR$, where $\domain\subset\bbR^d$, have prior
$f(\cdot)\sim\GP(\mu_{\bftheta}(\cdot),k_{\bftheta}(\cdot,\cdot))$.
We assume a differentiable $\mu_{\bftheta}$ and a twice continuously differentiable $k_{\bftheta}$.
At each training input $\bx_j$, the observations consist of a function
value and its gradient with noise $y_j=f(\bx_j)+\epsilon_j^y, \
\bg_j=\nabla_{\bx}f(\bx_j)+\bfepsilon_j^g$, where $\epsilon_j^y\sim\normal(0,\sigma_y^2)$ and
$\bfepsilon_j^g\sim\normal(\bfzero,\sigma_g^2\bfLambda_{\bftheta})$,
independently across observations and independently of $f$.
Here $\bfLambda_{\bftheta}$ is the inverse squared lengthscale matrix
in \eqref{eq:prelim_stationary_kernel} below, so the gradient noise is
isotropic in the scaled coordinates of Section~\ref{sec:methods_directions};
exact gradients correspond to $\sigma_g^2=0$.
We consider stationary kernels parameterized as
\begin{equation}
k_{\bftheta}(\bx,\bx')
=\kappa_{\bftheta}\left((\bx-\bx')^\top
\bfLambda_{\bftheta}(\bx-\bx')\right),
\label{eq:prelim_stationary_kernel}
\end{equation}
where $\kappa_{\bftheta}$ defines the covariance function and $\bfLambda_{\bftheta}=\diag(\ell_1^{-2},\ldots,\ell_d^{-2})$ is a diagonal matrix representing inverse squared lengthscales, allowing isotropic or ARD kernels~\citep{Rasmussen2005GaussianPF}.

For $\by=(y_1,\ldots,y_n)^\top$ and
$\bg=(\bg_1^\top,\ldots,\bg_n^\top)^\top$, the joint distribution is
\begin{equation}
\begin{bmatrix}\by\\\bg\end{bmatrix}
\sim\normal\left(
\begin{bmatrix}\bfmu_y\\\bfmu_g\end{bmatrix},
\begin{bmatrix}
\bK_{yy}&\bK_{yg}\\
\bK_{gy}&\bK_{gg}
\end{bmatrix}\right),
\label{eq:prelim_joint_dgp}
\end{equation}
where $[\bK_{yy}]_{jj'}=k_{\bftheta}(\bx_j,\bx_{j'})
+\mathbbm{1}\{j=j'\}\sigma_y^2$.
Let $\boldsymbol{\delta}_{jj'}=\bx_j-\bx_{j'}$ and
$r_{jj'}=\boldsymbol{\delta}_{jj'}^\top
\bfLambda_{\bftheta}\boldsymbol{\delta}_{jj'}$.
Differentiation gives the cross-covariance
\begin{equation}
\begin{aligned}
\cov(\bg_j,y_{j'})
&=\nabla_{\bx_j}k_{\bftheta}(\bx_j,\bx_{j'})\\
&=2\kappa_{\bftheta}'(r_{jj'})
\bfLambda_{\bftheta}\boldsymbol{\delta}_{jj'},
\end{aligned}
\label{eq:prelim_derivative_covariances}
\end{equation}
and the gradient block $\cov(\bg_j,\bg_{j'})$ is given in
\eqref{eq:app_gg_covariance} in Appendix~\ref{app:reduced_covariances}.

\subsection{Conditioning Sets and Prediction}
\label{sec:prelim_reference}

For a target observation $y_i$ at $\bx_i$, let
$c=\{j_1,\ldots,j_m\}$ index the conditioning observations,
with $i\notin c$.
Denote the corresponding function values and gradients by
$\by_c=(y_{j_1},\ldots,y_{j_m})^\top$ and
$\bg_c=(\bg_{j_1}^\top,\ldots,\bg_{j_m}^\top)^\top$.
We predict $y_i$ given the observations
$\bz_c=(\by_c^\top,\bg_c^\top)^\top$.
The posterior predictive distribution of $y_i$ given $\bz_c$
is Gaussian with mean and variance
\begin{equation}
\begin{aligned}
\mu_{y_i\mid\bz_c}
&=\mu_{y_i}
+\bk_{y_i\bz_c}\bK_{\bz_c\bz_c}^{-1}
(\bz_c-\bfmu_{\bz_c}),\\
\sigma^2_{y_i\mid\bz_c}
&=\sigma^2_{y_i}
-\bk_{y_i\bz_c}\bK_{\bz_c\bz_c}^{-1}\bk_{\bz_c y_i},
\end{aligned}
\label{eq:prelim_posterior_updates}
\end{equation}
where $\bK_{\bz_c\bz_c}$ is assumed positive definite.

\paragraph{Vecchia approximation.}
Restricting the size of conditioning sets allows efficient approximation of the joint distribution.
For function observations alone, the Vecchia approximation takes the form
\begin{equation}
p_{\bftheta}(\by)
=\prod_{i=1}^n p_{\bftheta}(y_i\mid\by_{<i})
\approx\prod_{i=1}^n p_{\bftheta}(y_i\mid\by_{c(i)}),
\label{eq:prelim_vecchia}
\end{equation}
where $\by_{<i}=(y_1,\ldots,y_{i-1})^\top$ and
$c(i)\subset\{1,\ldots,i-1\}$ contains at most $m$ preceding indices
under a maximin ordering.

\subsection{Exact Gradient Reduction}
\label{sec:prelim_tera}

For a target input $\bx_i$ and conditioning set $c$,
TERA defines
\begin{equation}
\begin{aligned}
\bD_i&=[\bx_{j_1}-\bx_i,\ldots,\bx_{j_m}-\bx_i],\\
\bq_c&=(\bI_m\otimes\bD_i^\top)\bg_c\in\bbR^{m^2}.
\end{aligned}
\label{eq:prelim_tera_statistic}
\end{equation}
This replaces an $md$ gradient with at most $m^2$
directional derivatives along the columns of $\bD_i$.
For stationary kernels of the form \eqref{eq:prelim_stationary_kernel},
the following equality holds under this gradient noise model,
including exact gradients when $\sigma_g^2=0$:
\begin{equation}
p_{\bftheta}(y_i\mid\by_c,\bg_c)
=p_{\bftheta}(y_i\mid\by_c,\bq_c).
\label{eq:prelim_tera_exactness}
\end{equation}

For ordered conditioning sets $c(i)\subset\{1,\ldots,i-1\}$,
TERA uses these exact reductions within a Vecchia approximation:
\begin{equation}
\tilde p_{\bftheta}(\by\mid\bg)
=\prod_{i=1}^n
p_{\bftheta}(y_i\mid\by_{c(i)},\bq_{c(i)}).
\label{eq:prelim_tera_likelihood}
\end{equation}
For $|c(i)|\le m$, dense factorization of the reduced covariance
matrix costs $\order(m^6)$ per Vecchia factor.

\section{Methods}
\label{sec:method}
We 
approximate each conditional factor in
\eqref{eq:prelim_tera_likelihood} by keeping at most
two directions per observed gradient.

\subsection{Selecting Two Directions per Gradient}
\label{sec:methods_directions}

We define scaled inputs and corresponding gradients, $\tilde{\bx}_j=\bfLambda_{\bftheta}^{1/2}\bx_j,\
\tilde{\bg}_j=\bfLambda_{\bftheta}^{-1/2}\bg_j
$,
where the gradient transformation follows from the chain rule.
We omit the tildes on the scaled inputs and gradients
throughout Section~\ref{sec:method} for notational simplicity.

Fix $\bftheta$, a target $y_i$, and a conditioning set
$c=\{j_1,\ldots,j_m\}$.
To understand how the two directions contribute to prediction, consider the posterior updates:
\begin{equation}
\begin{aligned}
\mu_{y_i\mid\by_c,\bg_c}
&=\mu_{y_i\mid\by_c}
+\bk_{y_i\bg_c\mid\by_c}
\bK_{\bg_c\bg_c\mid\by_c}^{-1}
(\bg_c-\bfmu_{\bg_c\mid\by_c}),\\
\sigma^2_{y_i\mid\by_c,\bg_c}
&=\sigma^2_{y_i\mid\by_c}
-\bk_{y_i\bg_c\mid\by_c}
\bK_{\bg_c\bg_c\mid\by_c}^{-1}
\bk_{\bg_c y_i\mid\by_c}.
\end{aligned}
\label{eq:methods_gradient_updates}
\end{equation}
We examine
$\bk_{\bg_j y_i\mid\by_c} = \cov(\bg_j,y_i\mid\by_c)$ for each $j\in c$,
which describes the dependence between the target and
each gradient after conditioning on the function values.
Define
\begin{equation}
\begin{aligned}
\bw_i
&=\var(\by_c)^{-1}\cov(\by_c,y_i),\\
\eta_{jj'}
&=-2\kappa_{\bftheta}'\!\left(
\|\bx_j-\bx_{j'}\|^2
\right).
\end{aligned}
\label{eq:methods_weights}
\end{equation}
In the scaled coordinates,
\eqref{eq:prelim_derivative_covariances} gives
$\cov(\bg_j,y_{j'})
=\eta_{jj'}(\bx_{j'}-\bx_j)$.
Gaussian conditioning then decomposes the cross-covariance
into two components:
\begin{equation}
\begin{aligned}
\cov(\bg_j,y_i\mid\by_c)
&=\cov(\bg_j,y_i)-\cov(\bg_j,\by_c)\bw_i\\
&=\underbrace{\eta_{ji}(\bx_i-\bx_j)}_{\bu_{ij}}\\
&\quad-\underbrace{\sum_{a=1}^{m}
[\bw_i]_a\,\eta_{j j_a}
(\bx_{j_a}-\bx_j)}_{\bv_{ij}}.
\end{aligned}
\label{eq:methods_conditional_covariance}
\end{equation}
The first component is target-directed, while the second
is a weighted aggregate of neighbor-directed directions.
We retain the corresponding two directional derivatives
at each conditioning input:
\begin{equation}
\boxed{
\bs_{ij}
=
\begin{bmatrix}
\bu_{ij}^{\top}\bg_j\\
\bv_{ij}^{\top}\bg_j
\end{bmatrix}
\in\bbR^2.
}
\label{eq:methods_two_derivatives}
\end{equation}
Stacking these pairs gives the reduced statistic
\begin{equation}
\bs_c
=
\bigl(\bs_{i j_1}^{\top},\ldots,
\bs_{i j_m}^{\top}\bigr)^{\top}
\in\bbR^{2m}.
\label{eq:methods_retained_derivatives}
\end{equation}
Consequently, we approximate each Vecchia conditional factor by
\[
p_{\bftheta}(y_i\mid\by_c,\bg_c)
\approx
p_{\bftheta}(y_i\mid\by_c,\bs_c).
\]
For numerical stability, we normalize each nonzero direction
to unit Euclidean norm.
Figure~\ref{fig:methods_two_directions} provides empirical support for our direction choice, which yields the lowest KL divergence among the directions compared.

\begin{figure}[t]
\centering
\includegraphics[width=0.96\linewidth]{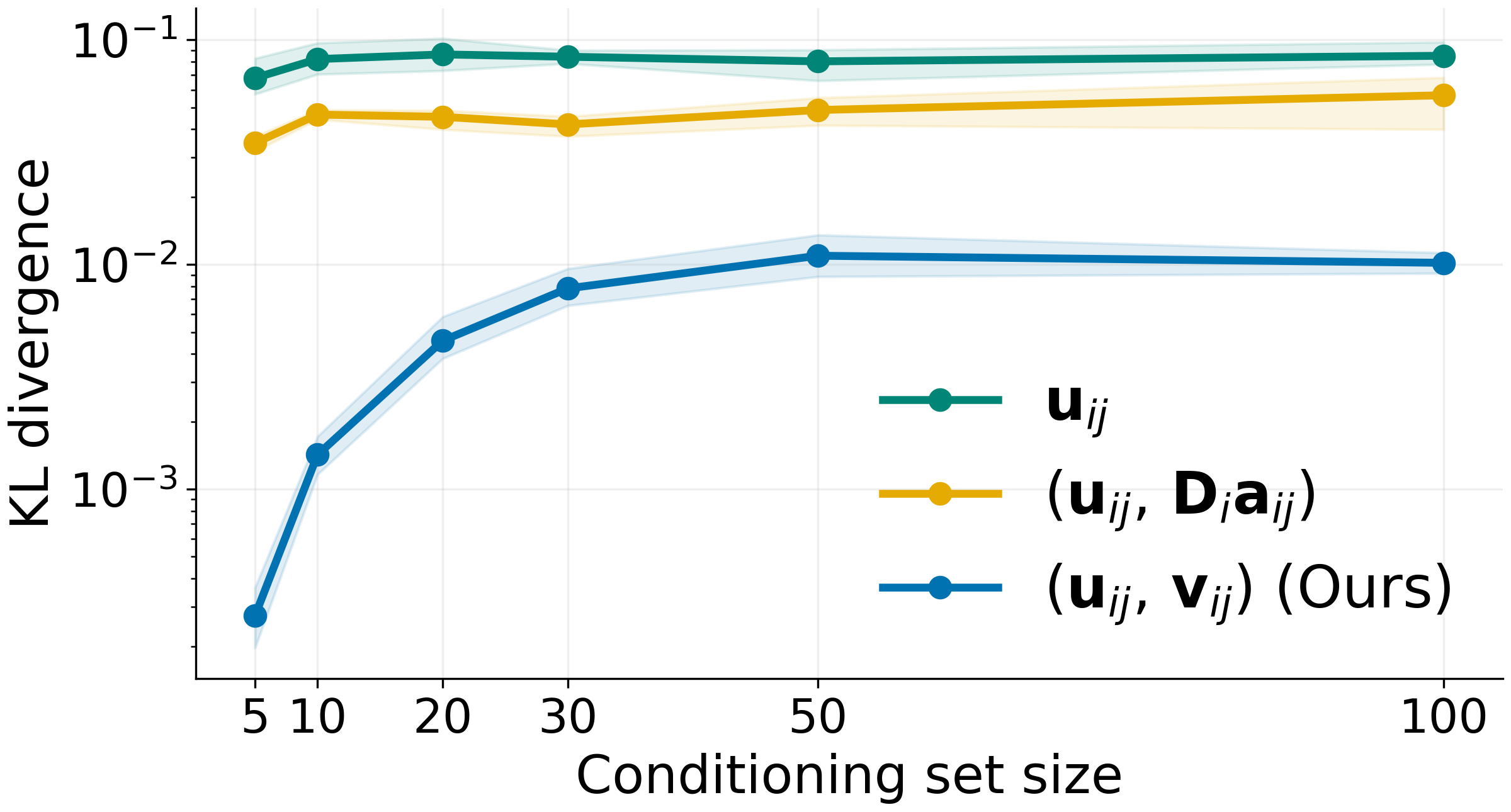}
\caption{Posterior fidelity under different directional-derivative choices as $m$ increases for a Mat\'ern-$5/2$ GP in $d=50$.
We compare \lite with methods using $\bu_{ij}$ 
or $\bu_{ij}$ paired with a random
direction in $\operatorname{col}(\bD_i)$.
KL divergence between the TERA and approximate posteriors is averaged over prediction inputs,
with all methods using the same conditioning inputs.}
\label{fig:methods_two_directions}
\end{figure}

\paragraph{Relation to TERA.} TERA's directional derivatives $\bD_i^\top\bg_j$ can be
equivalently expressed as target-directed and neighbor-directed
derivatives.
For each $j_a\in c$,
\[
(\bx_{j_a}-\bx_i)^\top\bg_j
=\underbrace{(\bx_{j_a}-\bx_j)^\top\bg_j}
_{\text{neighbor-directed}}
-\underbrace{(\bx_i-\bx_j)^\top\bg_j}
_{\text{target-directed}}.
\]
The entry with $j_a=j$ recovers the target-directed
derivative.
Our construction retains the target-directed component
and replaces the neighbor-directed derivatives with the
weighted aggregate identified by~\eqref{eq:methods_conditional_covariance}.
As in TERA, the gradient noise covariance is $\sigma_g^2\bI_d$ in the scaled coordinates.

\subsection{Exactness and Approximation Error}
\label{sec:methods_error}
We 
compare conditioning on $\bs_c$ with conditioning on $\bg_c$, holding the
conditioning set $c$ fixed.
We first establish exactness under a symmetry condition
and then analyze the error for general conditioning sets.

\paragraph{Exactness under compound symmetry.}
In high dimensions, pairwise distances among a fixed
number of independently sampled inputs become
approximately equal after rescaling under suitable distributional
assumptions~\citep{Beyer1999WhenI,Hall2005GeometricRO,Ahn2007TheHL}.
We consider the limiting case in which the conditioning
inputs are equally distant from the target and from
each other in scaled coordinates, so that the Gram matrix
$\bD_i^\top\bD_i$ has compound symmetry.
\begin{proposition}[Exactness under compound symmetry]
\label{prop:methods_symmetric_exactness}
Let $m\ge2$ and let $\bD_i$ contain the target-centered
differences in scaled coordinates.
Suppose that, for some $a,b\in\bbR$,
\[
\bD_i^\top\bD_i
=a\bI_m+b\mathbf1\mathbf1^\top.
\]
Under the assumed gradient noise model, assume that the
joint covariance of $(y_i,\by_c,\bg_c)$
is positive definite and that $\bu_{ij}$ and $\bv_{ij}$
are linearly independent for each $j\in c$.
Then
\[
p_{\bftheta}(y_i\mid\by_c,\bg_c)
=
p_{\bftheta}(y_i\mid\by_c,\bs_c).
\]
\end{proposition}
Under the stated assumptions, each gradient block of
$\bK_{\bg_c\bg_c\mid\by_c}^{-1}
\bk_{\bg_c y_i\mid\by_c}$ lies in $\operatorname{span}\{\bu_{ij},\bv_{ij}\}$.
Consequently, $\bs_c$ contains all the gradient information
needed to recover the posterior mean and variance in
\eqref{eq:methods_gradient_updates}.
The proof is given in Appendix~\ref{app:symmetric_exactness},
which also shows that the assumptions force $\bD_i^\top\bD_i$
to be positive definite, so the proposition requires
$m\le d$.

\begin{figure}[t]
\centering
\includegraphics[width=0.99\linewidth]{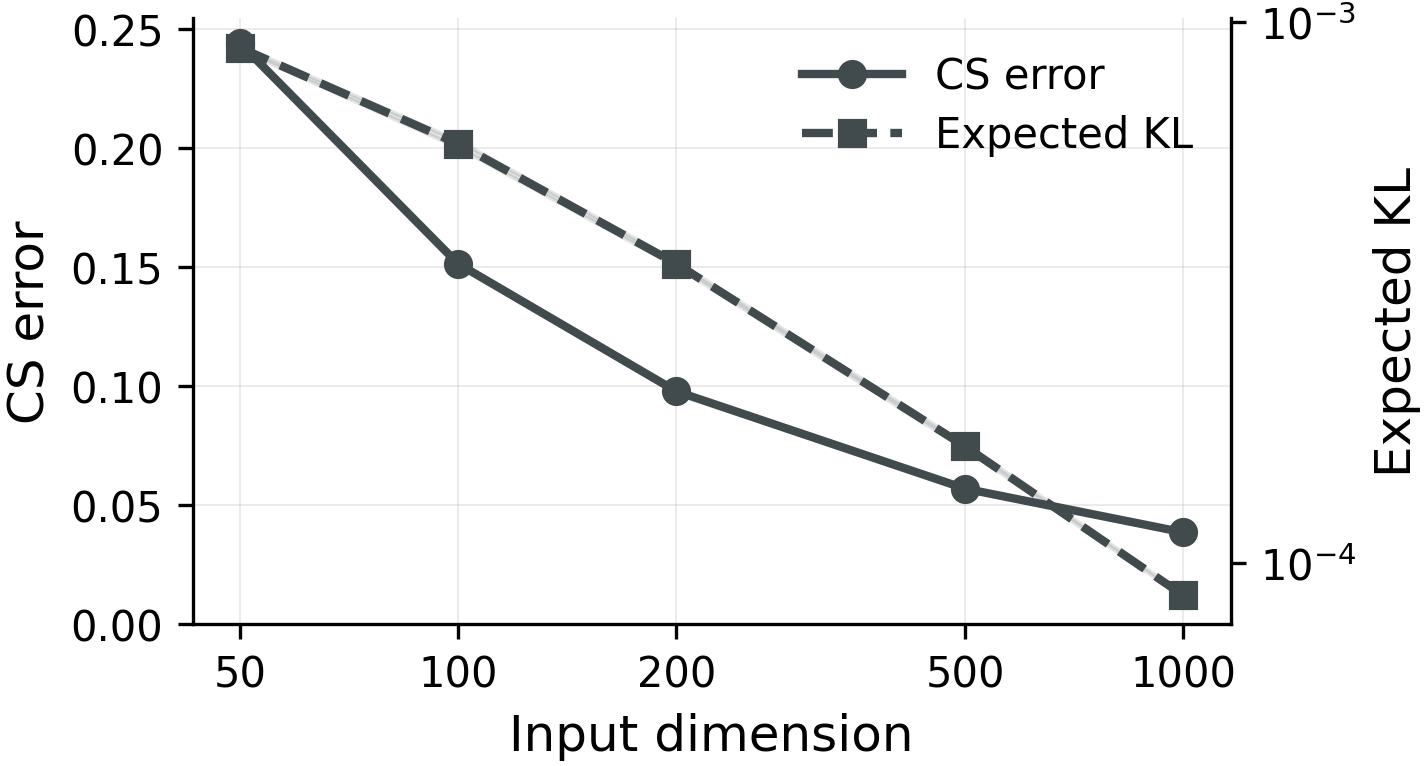}
\caption{Relative compound symmetry (CS) error (solid) and expected KL divergence from the full-gradient posterior to \lite's (dashed) as $d$ increases. We use an isotropic Matérn-$5/2$ kernel, $n=1024$ uniform inputs in $[0,1]^d$, $200$ targets, maximin ordering, and $m=30$ nearest preceding inputs.}
\label{fig:methods_symmetry}
\end{figure}

Vecchia conditioning sets can move toward this geometry.
Selecting the nearest inputs makes their distances to the
target more nearly equal than for independently sampled
inputs, and a maximin
ordering~\citep{Guinness2018PermutationAG,Schafer2020}
further equalizes the distances among the conditioning
inputs of the likelihood factors.
Such an ordering places each input greedily as far as
possible from all preceding inputs, so the inputs preceding
a target are spread as evenly as the greedy construction
allows, and no two of its conditioning inputs are closer
to each other than the target is to its nearest one.
Compound symmetry nevertheless holds only approximately;
Figure~\ref{fig:methods_symmetry} empirically shows that these conditioning sets approach compound symmetry as the input dimension increases, while the expected KL divergence from the full-gradient conditional distribution to \lite's in~\eqref{eq:methods_mean_kl_error} decreases. We quantify the deviation from compound symmetry by the
Frobenius distance of $\bD_i^\top\bD_i$ from its closest
compound-symmetry matrix, normalized by
$\|\bD_i^\top\bD_i\|_F$.

\paragraph{Posterior approximation error.}
For a general conditioning set, replacing $\bg_c$ with $\bs_c$ can change the posterior distribution of $y_i$.
Let
\[
p_i^F=p_{\bftheta}(y_i\mid\by_c,\bg_c),
\qquad
p_i^R=p_{\bftheta}(y_i\mid\by_c,\bs_c),
\]
with posterior means $\mu_i^F,\mu_i^R$ and variances $V_i^F,V_i^R$, respectively.
In this analysis, we assume that the inputs, conditioning set,
and covariance parameters are fixed, and the joint
covariance of $(y_i,\by_c,\bg_c)$ is positive definite.

Let $\bo_c$ collect the directional derivatives
along an orthonormal basis for the orthogonal complement
of $\operatorname{span}\{\bu_{ij},\bv_{ij}\}$ at each
conditioning input.
After redundant entries of $\bs_c$ are removed, the map
from $\bg_c$ to $(\bs_c,\bo_c)$ is invertible, so the pair
contains the same information as $\bg_c$.
By the definition of $\bo_c$ and \eqref{eq:methods_conditional_covariance},
$\cov(\bo_c,y_i\mid\by_c)=\bfzero$.
Under the Gaussian model, $\bo_c$ and $y_i$ are
therefore conditionally independent given $\by_c$.
This independence need not hold after conditioning
additionally on $\bs_c$, so omitting $\bo_c$
can change the target posterior.

Given $(\by_c,\bs_c)$, the reduced posterior is fixed,
whereas the full posterior additionally depends on
$\bo_c$.
We therefore quantify the information lost by omitting
$\bo_c$ through the expected KL divergence, where the
expectation is over $\bo_c\mid\by_c,\bs_c$.
Under the Gaussian model,
\begin{equation}
\bbE\left[
D_{\mathrm{KL}}(p_i^F\|p_i^R)
\mid\by_c,\bs_c
\right]
=\frac12\log\frac{V_i^R}{V_i^F}.
\label{eq:methods_mean_kl_error}
\end{equation}
Because \lite computes $V_i^R$ anyway, this error can be
evaluated exactly on a subsample of targets by computing
$V_i^F$ with TERA when $m\le d$, or with full gradients
otherwise.

To interpret the error, define
\begin{equation}
\begin{aligned}
\nu_i^y
&=
1-\frac{V_i^R}{\var(y_i\mid\by_c)},\\
\nu_i^\mu
&=
1-
\frac{
\var(\mu_i^R\mid\by_c,y_i,\bo_c)
}{
\var(\mu_i^R\mid\by_c,y_i)
},\\
\nu_i^\perp
&=
1-
\frac{
\var(\mu_i^R\mid\by_c,y_i,\bs_c^\perp,\bo_c)
}{
\var(\mu_i^R\mid\by_c,y_i)
},\\
\nu_i^{\max}
&=
\sup_{\bfomega\ne\bfzero}
\left\{
1-
\frac{
\var(\bfomega^\top\bs_c\mid\by_c,y_i,\bo_c)
}{
\var(\bfomega^\top\bs_c\mid\by_c,y_i)
}
\right\},
\end{aligned}
\label{eq:methods_error_quantities}
\end{equation}
where $\bfomega$ has the same dimension as $\bs_c$, and
$\bs_c^\perp=\bs_c-\cov(\bs_c,\mu_i^R\mid\by_c)\,
\mu_i^R/\var(\mu_i^R\mid\by_c)$ is the part of $\bs_c$
that is uncorrelated with both $\mu_i^R$ and $y_i$
given $\by_c$.
Here $0\le\nu_i^y<1$ measures how informative $\bs_c$ is
about the target beyond $\by_c$, and $\nu_i^{\max}<1$.
The quantities $\nu_i^\mu$ and $\nu_i^\perp$ are defined for
$\nu_i^y>0$ and measure how informative $\bo_c$ is about
the reduced posterior mean $\mu_i^R$ given $(\by_c,y_i)$,
without and with knowledge of $\bs_c^\perp$, while
$\nu_i^{\max}$ is the largest such measure over linear
combinations $\bfomega^\top\bs_c$.
These quantities require $\bo_c$, at $\order((md)^3)$ cost
per target, so they serve as diagnostics rather than
routine computations.

\begin{theorem}[Posterior approximation error]
\label{thm:methods_approximation_error}
Under the assumptions above, if $\nu_i^y=0$, then
$p_i^F=p_i^R$ for all values of $(\by_c,\bg_c)$.
If $\nu_i^y>0$, then $0\le\nu_i^\perp<1$ and
\begin{equation}
\bbE\left[
D_{\mathrm{KL}}(p_i^F\|p_i^R)
\mid\by_c,\bs_c
\right]
=
\frac12\log\left(
1+\frac{\nu_i^y\nu_i^\perp}{1-\nu_i^\perp}
\right).
\label{eq:methods_kl_identity}
\end{equation}
In this case, $p_i^F=p_i^R$ for all values of $(\by_c,\bg_c)$
if and only if $\nu_i^\mu=0$, which is equivalent to
$\nu_i^\perp=0$.
\end{theorem}

\begin{corollary}[Two-sided bound]
\label{cor:methods_kl_bound}
If $\nu_i^y>0$, then $\nu_i^\mu\le\nu_i^\perp$ and
$\nu_i^\perp/(1-\nu_i^\perp)\le\nu_i^\mu/(1-\nu_i^{\max})$.
Consequently,
\begin{equation}
\begin{aligned}
\frac12\log\left(
1+\frac{\nu_i^y\nu_i^\mu}{1-\nu_i^\mu}
\right)
&\le
\bbE\left[
D_{\mathrm{KL}}(p_i^F\|p_i^R)
\mid\by_c,\bs_c
\right]\\
&\le
\frac12\log\left(
1+\frac{\nu_i^y\nu_i^\mu}{1-\nu_i^{\max}}
\right).
\end{aligned}
\label{eq:methods_kl_bound}
\end{equation}
\end{corollary}

The identity \eqref{eq:methods_kl_identity} separates the
error into how informative $\bs_c$ is about the target,
$\nu_i^y$, and how much of the remaining variation in the
reduced posterior mean the omitted derivatives explain,
$\nu_i^\perp$.
If $\nu_i^{\max}\le1-\epsilon$ for some $\epsilon>0$, the
upper bound in \eqref{eq:methods_kl_bound} is at most
$\frac12\log(1+\nu_i^\mu/\epsilon)\le\nu_i^\mu/(2\epsilon)$,
so $\nu_i^\mu\ll\epsilon$ guarantees a small error even
when $\bs_c$ explains most of the target variance given
$\by_c$.
The approximation is exact if and only if
$\cov(\bo_c,\mu_i^R\mid\by_c)=\bfzero$, which, for $\nu_i^y>0$, is equivalent
to $\nu_i^\mu=0$ because $\cov(\bo_c,y_i\mid\by_c)=\bfzero$;
compound symmetry
(Proposition~\ref{prop:methods_symmetric_exactness}) is one
such setting.
Conversely, the lower bound implies a large error when
$\nu_i^y$ is substantial and $\nu_i^\mu$ is close to one.
The proofs are given in Appendix~\ref{app:approximation_error}.

\subsection{Parameter Learning and Prediction}
\label{sec:methods_learning_prediction}

\paragraph{Conditional likelihood.}
We estimate the covariance parameters by using the conditional log-likelihood
\begin{equation}
\tilde\ell_n(\bftheta)
=
\sum_{i=1}^n
\log p_{\bftheta}
\bigl(y_i\mid\by_{c(i)},\bs_{c(i)}\bigr).
\label{eq:methods_conditional_likelihood}
\end{equation}
We use stochastic updates based on minibatches of conditional factors~\citep{Cao2022ScalableGR,Jimenez2022ScalableBO}.
For each sampled target, we recompute the two
directions at the current covariance parameters then hold them fixed in the
original coordinates during differentiation. 
The covariance blocks remain functions of $\bftheta$,
while the ordering and conditioning sets are held
fixed during each update~\citep{Katzfuss2020ScaledVA}.
With the directions fixed, each update differentiates
a well-defined Gaussian likelihood of fixed linear
statistics of the gradients, just as fixing the ordering
and conditioning sets gives a fixed Vecchia approximation
within each update.

\paragraph{Batched inference.}
For each prediction input $\bx_\star$, we select
$c(\star)$ using the fitted input metric and construct
the two directions.
Let
\[
\bz_\star^R
=
\begin{pmatrix}
\by_{c(\star)}\\
\bs_{c(\star)}
\end{pmatrix},
\quad
\bK_\star=\cov(\bz_\star^R),
\quad
\bk_\star=\cov(\bz_\star^R,f_\star),
\]
where all moments are evaluated at $\hat{\bftheta}$.
Both directions lie in the column space of $\bD_\star$,
so $\bK_\star$ and $\bk_\star$ can be assembled directly
from $\bH_\star=\bD_\star^\top\bD_\star$.

For a batch $\mathcal B$ 
of $B$ targets, each with $m$ conditioning inputs, 
we stack the $B$ matrices
$\bK_\star$ of order $3m$. 
Batched Cholesky factorizations and triangular solves then
give all predictive means and variances in
\eqref{eq:prelim_posterior_updates} at once. 
Thus, covariance assembly and prediction can be parallelized across targets using matrices of order $3m$, allowing larger prediction batches within a fixed memory budget. 
The covariance construction and batched formulas are given in Appendix~\ref{app:reduced_covariances}.

\paragraph{Computational complexity.}
For each target, constructing the Gram matrix and
computing the directional derivatives cost
$\order(dm^2)$, while direction coefficient computation,
covariance assembly, and factorization cost
$\order(m^3)$.
A batch of $B$ targets therefore requires
\[
\order\bigl(B(dm^2+m^3)\bigr)
\ \ \text{time},
\quad
\order\bigl(B(dm+m^2)\bigr)
\ \ \text{memory}.
\]
Summed over the $n$ factors of
\eqref{eq:methods_conditional_likelihood}, one likelihood
evaluation costs $\order(n(dm^2+m^3))$, and the $dm^2$
assembly term dominates when $d>m$.
These costs exclude full dataset storage, ordering,
and conditioning set construction.
Table~\ref{tab:methods_cost} compares them with TERA and
with full gradients.

\begin{table}[t]
\centering
\caption{Per-target cost of one Vecchia conditional factor
with $m$ conditioning inputs in $d$ dimensions.
TERA's entries assume $m\le d$; when $m>d$, the full-gradient factor is smaller.
}
\label{tab:methods_cost}
\small
\setlength{\tabcolsep}{5pt}
\begin{tabular}{lccc}
\toprule
 & Factor size & Time & Memory \\
\midrule
Full grad & $m(d+1)$ & $\order(m^3d^3)$ & $\order(m^2d^2)$ \\
TERA & $m^2+m$ & $\order(dm^2+m^6)$ & $\order(dm^2+m^4)$ \\
\lite & $3m$ & $\order(dm^2+m^3)$ & $\order(dm+m^2)$ \\
\bottomrule
\end{tabular}
\end{table}

\section{Experiments}
\label{sec:experiment}
\begin{figure*}[t]
\centering
\includegraphics[width=0.72\linewidth]{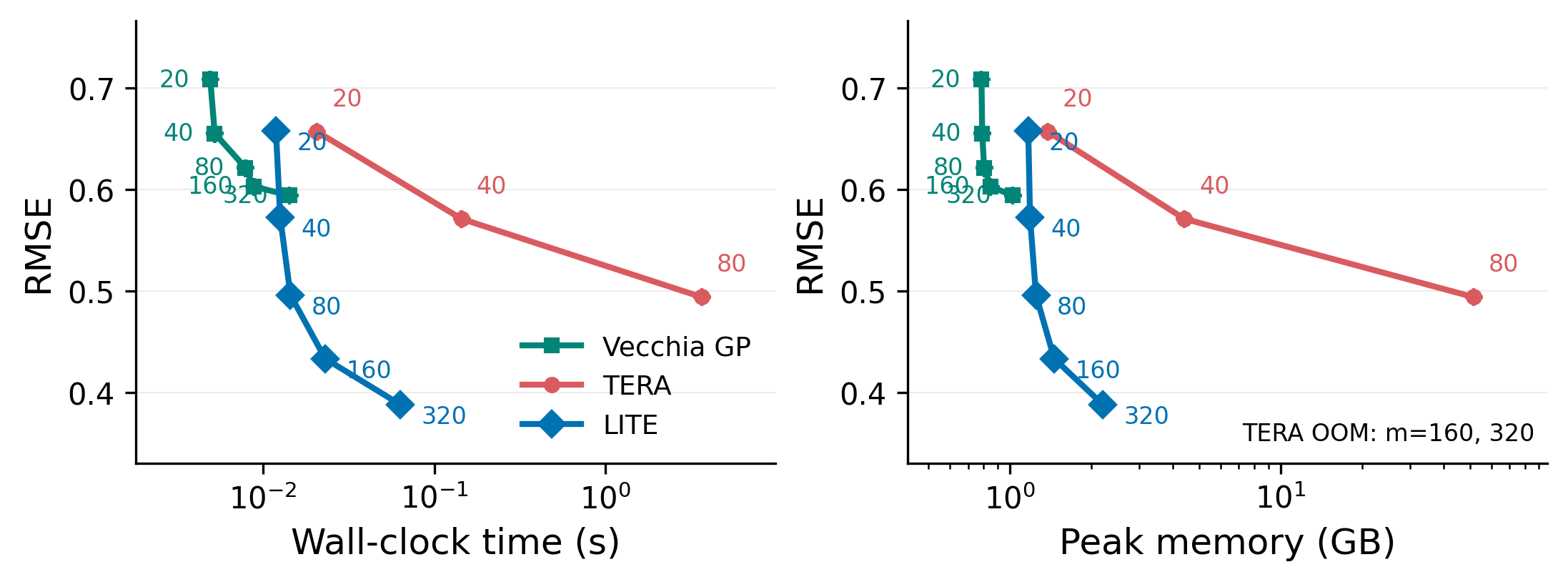}
\includegraphics[width=0.72\linewidth]{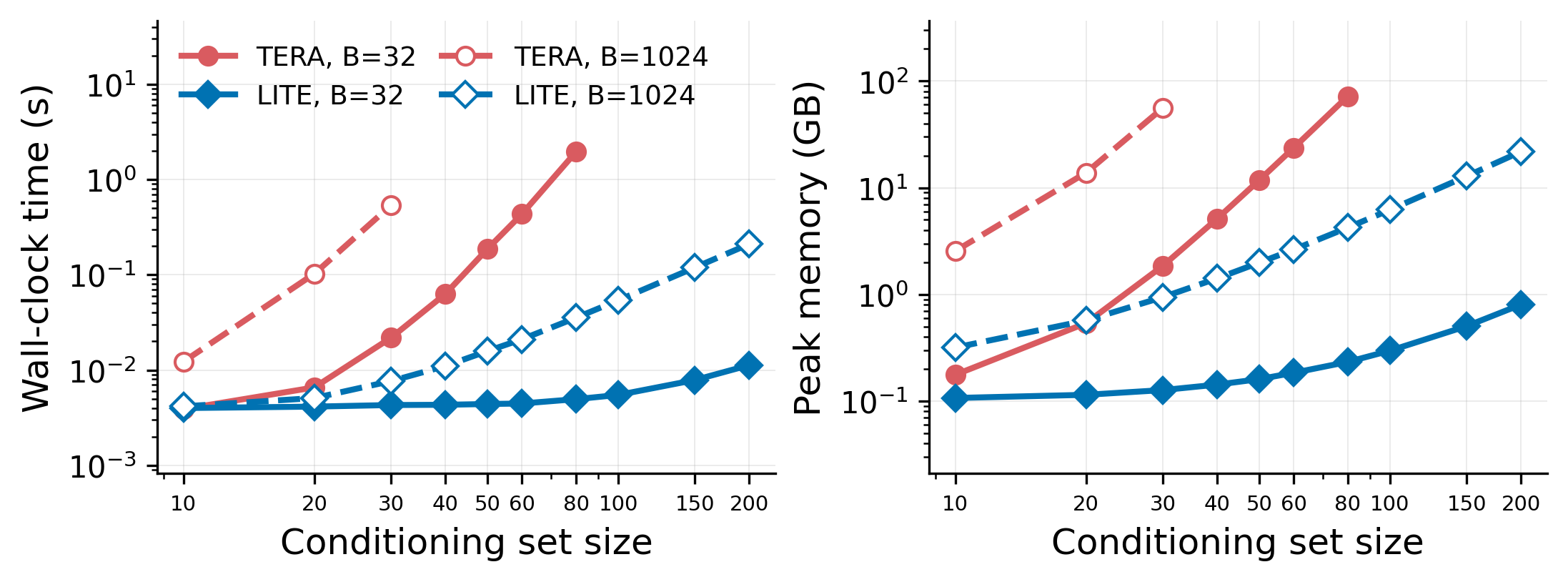}
\caption{RMSE under the assumed GP against total prediction time for $256$ targets and peak GPU memory (top). Wall-clock time and peak GPU memory of one training step against $m$ on the Buckyball-catcher dataset (bottom). Point labels in the top panels indicate $m$.
Prediction uses $B=32$, while the bottom panels use
$B=32$ and $B=1{,}024$.
TERA exceeds the $80$ GB GPU memory for $m\ge160$
in prediction; missing TERA points in the bottom panels exceed the GPU memory.}
\label{fig:experiment_scaling}
\end{figure*}

\begin{figure*}[t]
\centering
\includegraphics[width=0.99\linewidth]{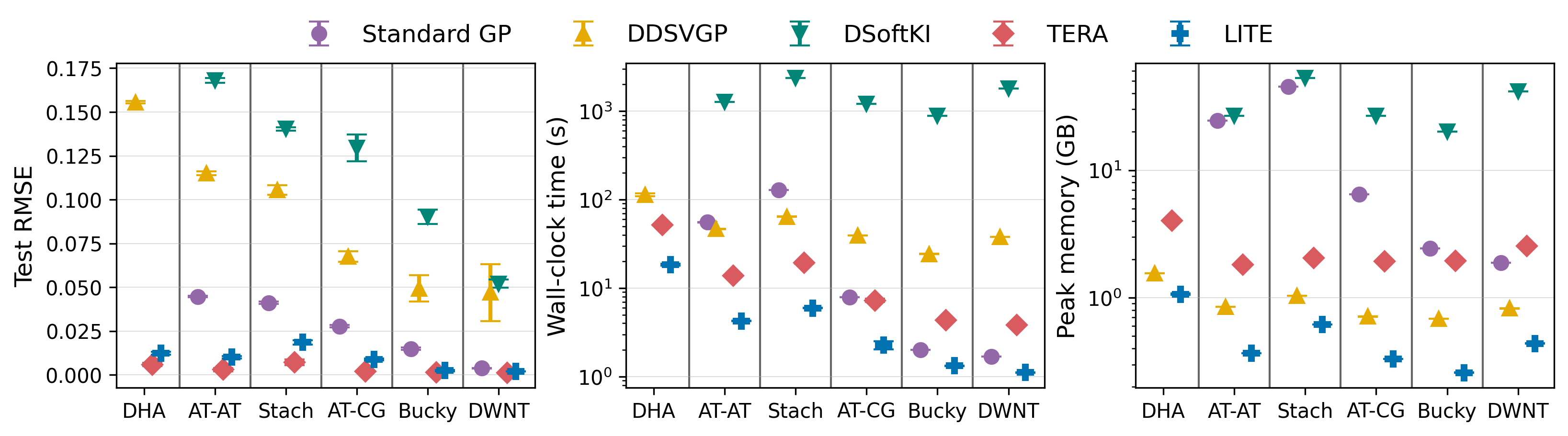}
\caption{Test RMSE per atom, wall-clock time, and peak GPU memory on six MD22 molecular datasets. \lite achieves the lowest runtime and memory usage on every benchmark, running $2.8$--$3.5\times$ faster than TERA while using only $13.3$--$30.2\%$ of its memory. Its RMSE is higher than TERA's but lower than that of the other methods. Standard GP and DSoftKI run out of memory on an 80 GB GPU when applied to DHA, whereas \lite requires only $1.06$ GB. Results are averaged over five independent runs, with error bars indicating standard errors.}
\label{fig:experiment_regression_md22}
\end{figure*}

We evaluate whether two directions per gradient retain
useful predictive information at reduced cost, and whether
these savings allow larger conditioning sets to improve
the accuracy-cost tradeoff. We assess prediction under a known GP, regression on molecular data, and high-dimensional BO.
We consider settings where gradient observations are available in existing datasets~\citep{Chmiela2016MachineLO} or obtained by differentiating the underlying models~\citep{Baydin2015AutomaticDI,Bertrand2020ImplicitDO}.
Our implementation uses PyTorch~\citep{Paszke2019PyTorchAI},
with GPyTorch~\citep{Gardner2018GPyTorchBM} and
BoTorch~\citep{Balandat2019BoTorchPB} where applicable.
All experiments run on a single NVIDIA H100 GPU. 
We use double precision for all methods except DSoftKI,
which uses single precision. To compare algorithmic costs under a common batching strategy, we replace TERA's target-wise sequential computations with batched training and prediction, using the same batch sizes as \lite. Unless otherwise specified, they share the same experimental settings. Experimental details and additional results 
are provided in Appendices~\ref{app:experiment_settings} and
\ref{app:additional_results}, 
respectively.

\begin{figure*}[t]
\centering
\includegraphics[width=0.99\linewidth]{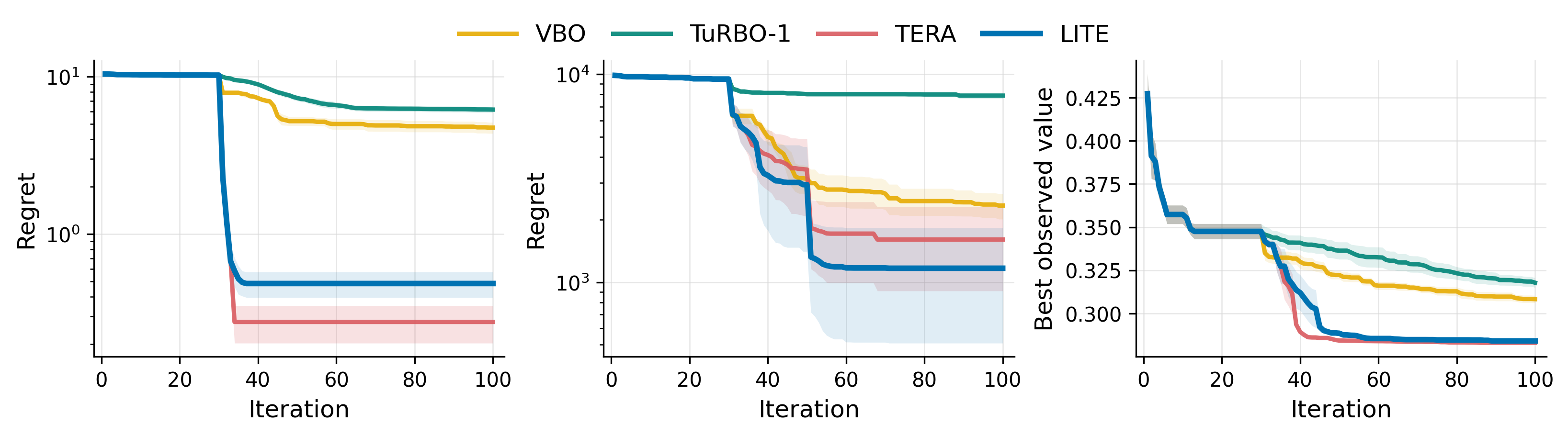}
\caption{Simple regret on Ackley ($500D$, left), Levy ($800D$, center), and best observed value on LassoDNA ($180D$, right) versus iteration. With only two directions per gradient, \lite substantially outperforms the function-only baselines, VBO and TuRBO-1, on all three tasks. It achieves the lowest final regret on Levy and a final value comparable to TERA on LassoDNA, while TERA attains lower regret on Ackley. Curves average five independent runs, with shaded bands indicating standard errors.}
\label{fig:experiment_bo}
\end{figure*}

\subsection{Scaling to Larger Conditioning Sets and Batch Sizes}
\label{sec:experiments_scaling}
We first examine whether the lower cost of \lite lets larger conditioning sets improve accuracy.
We evaluate prediction under an ARD Mat\'ern-$5/2$ GP
with $n=100,000$ inputs in $[0,1]^{500}$ and substantial
function and gradient observation noise.
This noise weakens the screening effect of nearby
observations, allowing predictive accuracy to continue
improving as $m$ grows.
With covariance parameters fixed, we compute the RMSE analytically over $256$ targets, avoiding
joint sampling at this scale.
We vary $m\in\{20,40,80,160,320\}$.
Evaluating prediction error under the assumed GP with fixed covariance parameters 
isolates the effect of the conditioning set size from model misspecification and parameter estimation.

\paragraph{Larger conditioning sets at a fraction of the cost.}
At equal conditioning set size, \lite matches the accuracy of TERA (top panels of Figure~\ref{fig:experiment_scaling}), so the two-direction budget costs little accuracy in this setting.
The error of \lite keeps decreasing up to $m=320$, whereas TERA exceeds the GPU memory for $m\ge160$ at the tested batch size of $32$.
At $m=320$, \lite attains a lower RMSE than TERA at $m=80$, approximately $0.39$ versus $0.49$, while running $58\times$ faster and using only $4.3\%$ of its memory. 
Thus, when additional conditioning inputs improve
prediction, \lite's lower cost allows larger conditioning
sets to outweigh the information lost by the reduction. Both derivative methods also achieve lower error than
the function-only Vecchia GP at equal $m$, showing that
gradients remain informative in this noisy setting.

\paragraph{Cost per training step.}
The bottom panels of Figure~\ref{fig:experiment_scaling} show the cost of one training step on the MD22 Buckyball-catcher dataset.
With batches of $32$ targets, the step time of \lite stays nearly constant up to $m=100$ because fixed per-step costs dominate its small factorizations, whereas the step time of TERA grows steeply as its $\order(m^6)$ factorization takes over, and TERA exceeds the GPU memory for $m\ge100$.
With batches of $1{,}024$ targets, \lite is about $70\times$ faster than TERA already at $m=30$, and TERA exceeds the GPU memory for larger $m$.
These fixed per-step costs also explain the moderate speedups at $m=30$ with batches of $32$ in the experiments below.

\subsection{Large-Scale GP Regression}
\label{sec:experiment_regression}
For GP regression, we use six MD22 molecular benchmarks~\citep{Chmiela2022AccurateGM}, with up to $69{,}753$ samples and $1{,}110$ input dimensions. We predict held-out molecular energies from atomic configurations, using observed forces as gradient information.
We compare \lite with three derivative GP methods, TERA, DDSVGP, DSoftKI, and a function-only Standard GP.
Figure~\ref{fig:experiment_regression_md22} compares test RMSE per atom, wall-clock time, and peak GPU memory.

\paragraph{Competitive accuracy with lowest computational costs.}
\lite achieves substantially lower RMSE than DDSVGP and DSoftKI and outperforms the function-only Standard GP on every benchmark.
Across all six benchmarks, it achieves the lowest runtime and
peak GPU memory usage. These savings also extend to the function-only Standard GP, showing that incorporating gradient information need not increase computational cost. 
Despite retaining only two directions per gradient, it achieves predictive accuracy slightly lower than TERA on several benchmarks. The comparison with TERA illustrates the tradeoff between exact gradient reduction and a fixed two-direction budget. Appendix~\ref{app:md22_shared_parameters} examines whether
shared covariance parameters, larger conditioning sets,
or further training resolve this accuracy gap.

\subsection{High-Dimensional Bayesian Optimization}
We evaluate BO on Ackley and Levy in $500$ and $800$ dimensions, respectively, and on the real-world LassoDNA~\citep{Sehic2022LassoBench} hyperparameter optimization task in $180$ dimensions.
We compare \lite with the derivative GP method TERA and the function-only methods VBO~\citep{Hvarfner2024VanillaBO} and TuRBO-1~\citep{Eriksson2019ScalableGO}.
Figure~\ref{fig:experiment_bo} compares optimization performance.

\paragraph{Effective optimization with two directions.}
On both synthetic tasks, \lite achieves substantially lower regret than the function-only baselines. It attains the lowest final regret on Levy, while TERA achieves
lower regret on Ackley.
On LassoDNA, \lite reaches a final
objective value comparable to TERA and substantially lower
than both function-only baselines. The advantage over function-only baselines thus extends
beyond synthetic objectives to a real-world optimization task.
These results show that two directions per gradient
preserve useful information for effective optimization in the tested $180$--$800$-dimensional tasks.
Runtime and memory comparisons are provided in
Appendix~\ref{app:bo_computational_cost}.

\section{Conclusions}
\label{sec:conclusion}
We proposed \lite, a derivative GP that retains two directional derivatives per observed gradient. 
Within a Vecchia approximation, this reduces the per-target factorization cost to $\order(m^3)$ and the covariance memory to $\order(m^2)$, which permits much larger conditioning sets and batched prediction over many targets.
The reduction is exact under compound symmetry, and for general conditioning sets we characterized the expected KL divergence from the posterior using full gradients through interpretable variance ratios.
Empirically, \lite matches the accuracy of TERA at equal conditioning set size in simulations and, with larger conditioning sets beyond TERA's memory limit, reaches lower error at a small fraction of TERA's time and memory.
On the GP regression and high-dimensional BO benchmarks, it remains competitive while using substantially less memory.

\bibliographystyle{plainnat}
\bibliography{main}

\clearpage
\appendix
\thispagestyle{empty}

\onecolumn
\aistatstitle{Appendix}
\section{Proofs for Section~\ref{sec:method}}

\subsection{Proof of Proposition~\ref{prop:methods_symmetric_exactness}}
\label{app:symmetric_exactness}

\begin{proof}
Recall that
\[
\bD_i
=
[\bx_{j_1}-\bx_i,\ldots,\bx_{j_m}-\bx_i].
\]
In the original coordinates, the Gram condition is
\[
\bD_i^\top\bfLambda_{\bftheta}\bD_i
=a\bI_m+b\mathbf1\mathbf1^\top.
\]
The transformation
$\tilde{\bx}_j=\bfLambda_{\bftheta}^{1/2}\bx_j$ and
$\tilde{\bg}_j=\bfLambda_{\bftheta}^{-1/2}\bg_j$
is invertible and leaves the directional derivative
statistics unchanged.
It therefore suffices to prove the result in scaled
coordinates, where the Gram condition becomes
$\tilde{\bD}_i^\top\tilde{\bD}_i
=a\bI_m+b\mathbf1\mathbf1^\top$.
As in Section~\ref{sec:methods_directions}, we omit
the tildes here for notational simplicity.

We will show that each gradient block of
\[
\bK_{\bg_c\bg_c\mid\by_c}^{-1}
\bk_{\bg_c y_i\mid\by_c}
\]
lies in $\operatorname{span}\{\bu_{ij},\bv_{ij}\}$.
By \eqref{eq:methods_gradient_updates}, this implies
that the full conditional mean depends on the gradients
only through $\bs_c$.

The assumed identity
$\bD_i^\top\bD_i=a\bI_m+b\mathbf1\mathbf1^\top$
implies that
\begin{equation}
\begin{aligned}
\|\bx_j-\bx_i\|^2&=a+b,
&&j\in c,\\
\|\bx_j-\bx_{j'}\|^2&=2a,
&&j,j'\in c,\quad j\ne j',
\end{aligned}
\label{eq:app_equal_distance}
\end{equation}
where the second equality follows from
\[
\begin{aligned}
\|\bx_j-\bx_{j'}\|^2
&=\|(\bx_j-\bx_i)-(\bx_{j'}-\bx_i)\|^2\\
&=\|\bx_j-\bx_i\|^2+\|\bx_{j'}-\bx_i\|^2-2(\bx_j-\bx_i)^\top(\bx_{j'}-\bx_i)\\
&=(a+b)+(a+b)-2b
=2a.
\end{aligned}
\]
Thus all conditioning inputs have the same distance
to the target, and all distinct conditioning inputs
have the same distance to one another.

Let $\bar{\bx}_c=m^{-1}\sum_{j\in c}\bx_j$.
By definition,
\[
\sum_{j\in c}(\bx_j-\bar{\bx}_c)
=\sum_{j\in c}\bx_j-m\bar{\bx}_c
=\bfzero.
\]
The Gram condition gives
\[
(\bx_j-\bx_i)^\top(\bx_{j'}-\bx_i)
=a\mathbbm{1}\{j=j'\}+b.
\]
Averaging over $j'\in c$ yields
\[
\begin{aligned}
(\bx_j-\bx_i)^\top(\bar{\bx}_c-\bx_i)
&=\frac1m\sum_{j'\in c}
(\bx_j-\bx_i)^\top(\bx_{j'}-\bx_i)\\
&=\frac{a+mb}{m}
=\frac am+b.
\end{aligned}
\]
Since this holds for every $j\in c$, averaging again
over $j$ gives
\[
\begin{aligned}
\|\bar{\bx}_c-\bx_i\|^2
&=\frac1m\sum_{j\in c}
(\bx_j-\bx_i)^\top(\bar{\bx}_c-\bx_i)\\
&=\frac am+b.
\end{aligned}
\]

We now use
$\bx_j-\bar{\bx}_c
=(\bx_j-\bx_i)-(\bar{\bx}_c-\bx_i)$
to compute the inner products after centering.
For any $j,j'\in c$,
\[
\begin{aligned}
(\bx_j-\bar{\bx}_c)^\top(\bx_{j'}-\bar{\bx}_c)
&=(\bx_j-\bx_i)^\top(\bx_{j'}-\bx_i)
-(\bx_j-\bx_i)^\top(\bar{\bx}_c-\bx_i)
-(\bar{\bx}_c-\bx_i)^\top(\bx_{j'}-\bx_i)
+\|\bar{\bx}_c-\bx_i\|^2\\
&=a\mathbbm{1}\{j=j'\}+b
-2\left(\frac am+b\right)+\left(\frac am+b\right)\\
&=a\left(\mathbbm{1}\{j=j'\}-\frac1m\right).
\end{aligned}
\]
Similarly,
\[
\begin{aligned}
(\bx_j-\bar{\bx}_c)^\top(\bar{\bx}_c-\bx_i)
&=(\bx_j-\bx_i)^\top(\bar{\bx}_c-\bx_i)
-\|\bar{\bx}_c-\bx_i\|^2\\
&=\left(\frac am+b\right)-\left(\frac am+b\right)
=0.
\end{aligned}
\]
Thus
\begin{equation}
\begin{aligned}
(\bx_j-\bar{\bx}_c)^\top(\bx_{j'}-\bar{\bx}_c)
&=a\left(\mathbbm{1}\{j=j'\}-\frac1m\right),\\
(\bx_j-\bar{\bx}_c)^\top(\bar{\bx}_c-\bx_i)&=0.
\end{aligned}
\label{eq:app_symmetric_geometry}
\end{equation}
In particular, each centered input difference is
orthogonal to the vector from the target to the
mean of the conditioning inputs.

Under the independent observation noise model,
\[
\cov(y_j,y_{j'})
=
\kappa_{\bftheta}\!\left(\|\bx_j-\bx_{j'}\|^2\right)
+\sigma_y^2\mathbbm{1}\{j=j'\}.
\]
Since $i\notin c$, \eqref{eq:app_equal_distance} gives
\[
\cov(y_j,y_i)
=\kappa_{\bftheta}\!\left(\|\bx_j-\bx_i\|^2\right)
=\kappa_{\bftheta}(a+b),
\qquad j\in c.
\]
Stacking these identical entries yields
\[
\cov(\by_c,y_i)=\kappa_{\bftheta}(a+b)\mathbf1.
\]
Similarly, for $j,j'\in c$,
\[
\cov(y_j,y_{j'})
=
\begin{cases}
\kappa_{\bftheta}(0)+\sigma_y^2, & j=j',\\
\kappa_{\bftheta}(2a), & j\ne j'.
\end{cases}
\]
Thus $\var(\by_c)$ can be written as
\[
\var(\by_c)
=
\bigl[\kappa_{\bftheta}(0)+\sigma_y^2
-\kappa_{\bftheta}(2a)\bigr]\bI_m
+\kappa_{\bftheta}(2a)\mathbf1\mathbf1^\top.
\]

Applying the Sherman-Morrison formula gives
\begin{equation}
\begin{aligned}
\var(\by_c)^{-1}
&=
\frac{1}
{\kappa_{\bftheta}(0)+\sigma_y^2-\kappa_{\bftheta}(2a)}\\
&\quad\times
\left[
\bI_m-
\frac{\kappa_{\bftheta}(2a)\mathbf1\mathbf1^\top}
{\kappa_{\bftheta}(0)+\sigma_y^2
+(m-1)\kappa_{\bftheta}(2a)}
\right].
\end{aligned}
\label{eq:yc_variance_inverse}
\end{equation}
Therefore, we obtain
\[
\begin{aligned}
\bw_i
&=\var(\by_c)^{-1}\cov(\by_c,y_i)\\
&=
\frac{\kappa_{\bftheta}(a+b)}
{\kappa_{\bftheta}(0)+\sigma_y^2-\kappa_{\bftheta}(2a)}
\left[
1-\frac{m\kappa_{\bftheta}(2a)}
{\kappa_{\bftheta}(0)+\sigma_y^2
+(m-1)\kappa_{\bftheta}(2a)}
\right]\mathbf1\\
&=
\frac{\kappa_{\bftheta}(a+b)}
{\kappa_{\bftheta}(0)+\sigma_y^2
+(m-1)\kappa_{\bftheta}(2a)}
\mathbf1.
\end{aligned}
\]

Since $\eta_{ji}=-2\kappa_{\bftheta}'(a+b)$ for every
$j\in c$, the first direction is
\[
\bu_{ij}
=-2\kappa_{\bftheta}'(a+b)(\bx_i-\bx_j).
\]
For the second direction, all entries of $\bw_i$ are equal
and $\eta_{jj'}=-2\kappa_{\bftheta}'(2a)$ whenever $j\ne j'$.
The term with $j'=j$ vanishes, so
\[
\begin{aligned}
\bv_{ij}
&=-2[\bw_i]_1\kappa_{\bftheta}'(2a)
\sum_{\substack{j'\in c\\j'\ne j}}(\bx_{j'}-\bx_j)\\
&=-2m[\bw_i]_1\kappa_{\bftheta}'(2a)
(\bar{\bx}_c-\bx_j).
\end{aligned}
\]

The expressions for the two directions and their assumed
linear independence imply
\[
\operatorname{span}\{\bu_{ij},\bv_{ij}\}
=
\operatorname{span}
\{\bx_j-\bar{\bx}_c,\bar{\bx}_c-\bx_i\},
\]
because
$\bx_i-\bx_j
=-(\bx_j-\bar{\bx}_c)-(\bar{\bx}_c-\bx_i)$.

Accordingly, let $\mathcal S$ be the space of stacked vectors
$\bh=(\bh_{j_1}^\top,\ldots,\bh_{j_m}^\top)^\top$
whose blocks satisfy
\[
\bh_j
=\alpha(\bx_j-\bar{\bx}_c)
+\beta(\bar{\bx}_c-\bx_i),
\qquad j\in c,
\]
for common coefficients $\alpha,\beta\in\bbR$.
Every block of a vector in $\mathcal S$ thus belongs
to the corresponding span of the two directions.

We will verify that
\begin{equation}
\bk_{\bg_c y_i\mid\by_c}\in\mathcal S,
\qquad
\bK_{\bg_c\bg_c\mid\by_c}\bh\in\mathcal S
\quad\text{for every }\bh\in\mathcal S.
\label{eq:app_symmetric_invariance}
\end{equation}
The second property, together with positive definiteness,
ensures that the inverse conditional covariance also maps
$\mathcal S$ into itself.
The two properties therefore place the required coefficient
vector in $\mathcal S$.

The expressions for the two directions, together with
$\cov(\bg_j,y_i\mid\by_c)=\bu_{ij}-\bv_{ij}$
in \eqref{eq:methods_conditional_covariance}, give
\[
\bk_{\bg_c y_i\mid\by_c}\in\mathcal S.
\]

For the second property, write
\[
\bK_{\bg_c\bg_c\mid\by_c}
=
\bK_{\bg_c\bg_c}
-\bK_{\bg_c\by_c}\bK_{\by_c\by_c}^{-1}
 \bK_{\by_c\bg_c}.
\]
We show that both terms on the right map $\mathcal S$
into itself.
Take $\bh\in\mathcal S$.
For $j\ne j'$, \eqref{eq:app_symmetric_geometry} gives
\[
\begin{aligned}
(\bx_j-\bx_{j'})^\top\bh_{j'}
&=
\alpha(\bx_j-\bx_{j'})^\top(\bx_{j'}-\bar{\bx}_c)\\
&\quad+
\beta(\bx_j-\bx_{j'})^\top(\bar{\bx}_c-\bx_i)\\
&=
\alpha\left[
(\bx_j-\bar{\bx}_c)^\top(\bx_{j'}-\bar{\bx}_c)
-\|\bx_{j'}-\bar{\bx}_c\|^2
\right]\\
&\quad+
\beta\left[
(\bx_j-\bar{\bx}_c)^\top(\bar{\bx}_c-\bx_i)
-(\bx_{j'}-\bar{\bx}_c)^\top(\bar{\bx}_c-\bx_i)
\right]\\
&=
\alpha\left[-\frac am-a\left(1-\frac1m\right)\right]
+\beta(0-0)\\
&=-\alpha a.
\end{aligned}
\]
Also,
\[
\begin{aligned}
\sum_{\substack{j'\in c\\j'\ne j}}\bh_{j'}
&=
\alpha\sum_{\substack{j'\in c\\j'\ne j}}
(\bx_{j'}-\bar{\bx}_c)
+(m-1)\beta(\bar{\bx}_c-\bx_i)\\
&=
\alpha\left[
\sum_{j'\in c}(\bx_{j'}-\bar{\bx}_c)
-(\bx_j-\bar{\bx}_c)
\right]
+(m-1)\beta(\bar{\bx}_c-\bx_i)\\
&=-\alpha(\bx_j-\bar{\bx}_c)
+(m-1)\beta(\bar{\bx}_c-\bx_i),\\[1ex]
\sum_{j'\in c}(\bx_j-\bx_{j'})
&=m(\bx_j-\bar{\bx}_c).
\end{aligned}
\]
Substituting these identities into
\eqref{eq:app_gg_covariance} yields
\[
\begin{aligned}
\sum_{j'\in c}\cov(\bg_j,\bg_{j'})\bh_{j'}&=\alpha\bigl[
-2\kappa_{\bftheta}'(0)+\sigma_g^2
+2\kappa_{\bftheta}'(2a)
+4ma\kappa_{\bftheta}''(2a)
\bigr](\bx_j-\bar{\bx}_c)\\
&\qquad+\beta\bigl[
-2\kappa_{\bftheta}'(0)+\sigma_g^2
-2(m-1)\kappa_{\bftheta}'(2a)
\bigr](\bar{\bx}_c-\bx_i).
\end{aligned}
\]
The coefficients are again independent of $j$, so
$\bK_{\bg_c\bg_c}\bh\in\mathcal S$.

It remains to show that the conditioning adjustment
\[
\bK_{\bg_c\by_c}\bK_{\by_c\by_c}^{-1}
\bK_{\by_c\bg_c}\bh
\]
belongs to $\mathcal S$.
Each entry corresponding to $y_j$ in
$\bK_{\by_c\bg_c}\bh$ is
\[
\begin{aligned}
\sum_{j'\in c}\cov(y_j,\bg_{j'})\bh_{j'}
&=-2\kappa_{\bftheta}'(2a)
\sum_{\substack{j'\in c\\j'\ne j}}
(\bx_j-\bx_{j'})^\top\bh_{j'}\\
&=2\alpha a(m-1)\kappa_{\bftheta}'(2a).
\end{aligned}
\]
Thus $\bK_{\by_c\bg_c}\bh$ is proportional to $\mathbf1$.
The Sherman-Morrison expression in
\eqref{eq:yc_variance_inverse} shows that
$\bK_{\by_c\by_c}^{-1}\mathbf1$ is also proportional
to $\mathbf1$.

Finally, each gradient block of
$\bK_{\bg_c\by_c}\mathbf1$ is
\[
\begin{aligned}
\sum_{j'\in c}\cov(\bg_j,y_{j'})
&=-2\kappa_{\bftheta}'(2a)
\sum_{j'\in c}(\bx_{j'}-\bx_j)\\
&=2m\kappa_{\bftheta}'(2a)(\bx_j-\bar{\bx}_c),
\end{aligned}
\]
where the term with $j'=j$ vanishes.
Combining these expressions shows that each block of
$\bK_{\bg_c\by_c}\bK_{\by_c\by_c}^{-1}
\bK_{\by_c\bg_c}\bh$ is a common scalar multiple of
$\bx_j-\bar{\bx}_c$.
Thus this vector belongs to $\mathcal S$.
Together with $\bK_{\bg_c\bg_c}\bh\in\mathcal S$,
this proves \eqref{eq:app_symmetric_invariance}.

By \eqref{eq:app_symmetric_invariance} and assumed positive
definiteness,
\[
\bK_{\bg_c\bg_c\mid\by_c}^{-1}
\bk_{\bg_c y_i\mid\by_c}\in\mathcal S.
\]
Hence its block corresponding to $\bg_j$ lies in
$\operatorname{span}\{\bu_{ij},\bv_{ij}\}$.
Each gradient then enters the conditional mean only through
$\bu_{ij}^\top\bg_j$ and $\bv_{ij}^\top\bg_j$,
with the centering terms determined by $\by_c$.

Thus the full conditional mean $\mu_{y_i\mid\by_c,\bg_c}$
is a function of $(\by_c,\bs_c)$ only.
Since $\bs_c$ is a function of $\bg_c$, the tower property
gives
\[
\mu_{y_i\mid\by_c,\bs_c}
=\bbE\bigl[\mu_{y_i\mid\by_c,\bg_c}\mid\by_c,\bs_c\bigr]
=\mu_{y_i\mid\by_c,\bg_c}.
\]
The full conditional variance $\sigma^2_{y_i\mid\by_c,\bg_c}$
does not depend on the observed values, so the law of total
variance gives
\[
\sigma^2_{y_i\mid\by_c,\bs_c}
=\bbE\bigl[\sigma^2_{y_i\mid\by_c,\bg_c}\mid\by_c,\bs_c\bigr]
+\var\bigl(\mu_{y_i\mid\by_c,\bg_c}\mid\by_c,\bs_c\bigr)
=\sigma^2_{y_i\mid\by_c,\bg_c}.
\]
Therefore, replacing $\bg_c$ with $\bs_c$ preserves
both the conditional mean and variance updates in
\eqref{eq:methods_gradient_updates}.
Both $p_{\bftheta}(y_i\mid\by_c,\bg_c)$ and
$p_{\bftheta}(y_i\mid\by_c,\bs_c)$ are Gaussian, so equality
of their means and variances proves the result.
\end{proof}

\paragraph{Positive definiteness of the Gram matrix.}
The assumptions of
Proposition~\ref{prop:methods_symmetric_exactness} imply that
$\bD_i^\top\bD_i=a\bI_m+b\mathbf1\mathbf1^\top$ is positive
definite, and hence that $m\le d$.
Its eigenvalues are $a$, with multiplicity $m-1$, and $a+mb$,
and both are nonnegative because $\bD_i^\top\bD_i$ is a Gram
matrix.
If $a=0$, then \eqref{eq:app_equal_distance} gives
$\bx_j=\bx_{j'}$ for all $j,j'\in c$, so $\bx_j=\bar{\bx}_c$
and the expression for $\bv_{ij}$ in the proof gives
$\bv_{ij}=\bfzero$.
If $a+mb=0$, then $\|\bar{\bx}_c-\bx_i\|^2=a/m+b=0$, so
$\bx_i=\bar{\bx}_c$, and $\bu_{ij}$ and $\bv_{ij}$ are both
multiples of $\bx_j-\bar{\bx}_c$.
Either case contradicts the linear independence of
$\bu_{ij}$ and $\bv_{ij}$, so both eigenvalues are positive.
A positive definite $m\times m$ Gram matrix of vectors in
$\bbR^d$ requires $m\le d$.

\subsection{Proof of Theorem~\ref{thm:methods_approximation_error} and Corollary~\ref{cor:methods_kl_bound}}
\label{app:approximation_error}

\begin{proof}
Throughout the proof, the inputs, conditioning set,
and covariance parameters are fixed.
We remove redundant entries of $\bs_c$, so that
the transformation from $\bg_c$ to $(\bs_c,\bo_c)$
is invertible.
The joint covariance of $(y_i,\by_c,\bs_c,\bo_c)$
is therefore positive definite.

By \eqref{eq:methods_conditional_covariance},
\[
\cov(\bg_j,y_i\mid\by_c)=\bu_{ij}-\bv_{ij}.
\]
Each omitted direction at input $\bx_j$ is orthogonal
to both $\bu_{ij}$ and $\bv_{ij}$.
Hence
\[
\cov(\bo_c,y_i\mid\by_c)=\bfzero.
\]
If $\nu_i^y=0$, the Gaussian conditional variance formula gives
\[
0
=\var(y_i\mid\by_c)-V_i^R
=\bk_{\bs_c y_i\mid\by_c}^{\top}
 \bK_{\bs_c\bs_c\mid\by_c}^{-1}
 \bk_{\bs_c y_i\mid\by_c}.
\]
Since $\bK_{\bs_c\bs_c\mid\by_c}$ is positive definite,
this implies $\cov(\bs_c,y_i\mid\by_c)=\bfzero$.
Thus $y_i$ is conditionally independent of $(\bs_c,\bo_c)$
given $\by_c$, and $p_i^F=p_i^R$ for all values of
$(\by_c,\bg_c)$.
If there are no omitted directions, then $p_i^F=p_i^R$
and $\nu_i^\mu=\nu_i^\perp=\nu_i^{\max}=0$, so all claims hold.
We henceforth assume $\nu_i^y>0$ and that $\bo_c$ is nonempty.

All remaining moments are conditional on $\by_c$, which we
suppress from the notation; for example, $\var(y_i)$ stands
for $\var(y_i\mid\by_c)$.
Write $V_0=\var(y_i)$, $\bK_{ss}=\var(\bs_c)$,
$\bk_s=\cov(\bs_c,y_i)$, $\bfbeta=\bK_{ss}^{-1}\bk_s$,
and $\mu=\mu_i^R$, so that
$\mu=\bbE[y_i]+\bfbeta^\top(\bs_c-\bbE[\bs_c])$ and
$V_i^R=V_0-\bk_s^\top\bfbeta$.

\paragraph{Step 1: moments of the reduced mean.}
We have
$\var(\mu)=\bfbeta^\top\bK_{ss}\bfbeta=\bk_s^\top\bfbeta
=V_0-V_i^R=\nu_i^yV_0>0$
and $\cov(\mu,y_i)=\bfbeta^\top\bk_s=\var(\mu)$.
Hence
\[
\var(\mu\mid y_i)
=\var(\mu)-\frac{\var(\mu)^2}{V_0}
=\nu_i^y(1-\nu_i^y)V_0.
\]

\paragraph{Step 2: decomposition of $\bs_c$.}
Let $\bA$ be a matrix whose rows form a basis of the
orthogonal complement of $\bk_s$, and let $\bt=\bA\bs_c$;
if $\bs_c$ has a single entry, $\bt$ is empty and the
statements involving $\bt$ below hold trivially.
Then $\cov(\bt,\mu)=\bA\bK_{ss}\bfbeta=\bA\bk_s=\bfzero$
and $\cov(\bt,y_i)=\bA\bk_s=\bfzero$.
Because $\bfbeta^\top\bk_s=\var(\mu)>0$, the vector $\bfbeta$
is not orthogonal to $\bk_s$, so $(\mu,\bt)$ is an invertible
affine function of $\bs_c$, and conditioning on $\bs_c$ is
equivalent to conditioning on $(\mu,\bt)$.
The vector $\bs_c^\perp$ in \eqref{eq:methods_error_quantities}
is an affine function of $\bs_c$ with linear part
$\bI-\bk_s\bfbeta^\top/\var(\mu)$, whose null space is spanned
by $\bk_s$, so its row space equals that of $\bA$, and
conditioning on $\bs_c^\perp$ is equivalent to conditioning
on $\bt$.
Since $\bt$ and $\bo_c$ are uncorrelated with $y_i$,
conditioning on $y_i$ leaves $\var(\bt)$, $\var(\bo_c)$,
$\cov(\bo_c,\bt)$, and $\cov(\bo_c,\mu)$ unchanged.

\paragraph{Step 3: the variance gap.}
Let $\bfgamma=\cov(\bo_c,\mu)=\cov(\bo_c,\bs_c)\bfbeta$ and
$\bP=\var(\bo_c\mid\bt)$, which is positive definite.
Since $\cov(\bo_c,y_i)=\bfzero$, we have
$\cov(\bo_c,y_i\mid\bs_c)=-\cov(\bo_c,\bs_c)\bfbeta=-\bfgamma$,
so
\[
V_i^R-V_i^F=\bfgamma^\top\var(\bo_c\mid\bs_c)^{-1}\bfgamma.
\]
Because $\cov(\mu,\bt)=\bfzero$, we have
$\cov(\bo_c,\mu\mid\bt)=\bfgamma$ and
$\var(\mu\mid\bt)=\var(\mu)$, so that
$\var(\bo_c\mid\bs_c)=\var(\bo_c\mid\mu,\bt)
=\bP-\bfgamma\bfgamma^\top/\var(\mu)$.
Let $\zeta=\bfgamma^\top\bP^{-1}\bfgamma$.
Positive definiteness of $\var(\bo_c\mid\bs_c)$ gives
$\zeta<\var(\mu)$, and the Sherman-Morrison formula gives
\[
V_i^R-V_i^F
=\zeta+\frac{\zeta^2}{\var(\mu)-\zeta}
=\frac{\zeta\,\var(\mu)}{\var(\mu)-\zeta}.
\]

\paragraph{Step 4: the identity.}
By Step 2, $\var(\mu\mid y_i,\bt)=\var(\mu\mid y_i)$,
$\cov(\mu,\bo_c\mid y_i,\bt)=\bfgamma^\top$, and
$\var(\bo_c\mid y_i,\bt)=\bP$, so
$\var(\mu\mid y_i,\bt,\bo_c)=\var(\mu\mid y_i)-\zeta$.
Since conditioning on $\bs_c^\perp$ is equivalent to
conditioning on $\bt$, the definition of $\nu_i^\perp$ and
Step 1 give
\[
\nu_i^\perp
=\frac{\zeta}{\var(\mu\mid y_i)}
=\frac{\zeta}{\nu_i^y(1-\nu_i^y)V_0},
\]
and $\nu_i^\perp<1$ because
$\var(\mu\mid y_i,\bt,\bo_c)>0$.
Substituting $\zeta=\nu_i^\perp\nu_i^y(1-\nu_i^y)V_0$,
$\var(\mu)=\nu_i^yV_0$, and $V_i^R=(1-\nu_i^y)V_0$ into
Step 3 gives
\[
V_i^R-V_i^F
=\frac{\nu_i^\perp\nu_i^y(1-\nu_i^y)V_0}{1-\nu_i^\perp(1-\nu_i^y)},
\qquad
V_i^F
=\frac{(1-\nu_i^y)(1-\nu_i^\perp)V_0}{1-\nu_i^\perp(1-\nu_i^y)},
\]
and hence
\[
\frac{V_i^R}{V_i^F}
=\frac{1-\nu_i^\perp(1-\nu_i^y)}{1-\nu_i^\perp}
=1+\frac{\nu_i^y\nu_i^\perp}{1-\nu_i^\perp}.
\]
The identity \eqref{eq:methods_kl_identity} now follows from
\eqref{eq:methods_mean_kl_error}.

\paragraph{Step 5: exact recovery.}
Since $\var(\bo_c\mid y_i)=\var(\bo_c)$ and
$\cov(\mu,\bo_c\mid y_i)=\bfgamma^\top$,
\[
\nu_i^\mu
=\frac{\bfgamma^\top\var(\bo_c)^{-1}\bfgamma}{\var(\mu\mid y_i)},
\]
so both $\nu_i^\mu$ and $\nu_i^\perp$ vanish if and only if
$\bfgamma=\bfzero$.
If $\bfgamma=\bfzero$, then $\zeta=0$ and Step 3 gives
$V_i^F=V_i^R$.
By the law of total variance,
$\var(\mu_i^F-\mu_i^R)=\var(\mu_i^F\mid\bs_c)=V_i^R-V_i^F=0$.
Since $\mu_i^F-\mu_i^R$ is an affine function of
$(\bs_c,\bo_c)$ with zero mean and zero variance, and
$(\bs_c,\bo_c)$ has a positive definite covariance matrix,
$\mu_i^F=\mu_i^R$ for all values of $(\by_c,\bg_c)$, and
hence $p_i^F=p_i^R$.
Conversely, if $p_i^F=p_i^R$ for all values of
$(\by_c,\bg_c)$, then $V_i^F=V_i^R$, so $\zeta=0$ by Step 3
and $\bfgamma=\bfzero$.
This completes the proof of
Theorem~\ref{thm:methods_approximation_error}.

\paragraph{Step 6: proof of Corollary~\ref{cor:methods_kl_bound}.}
Conditioning additionally on $\bt$ cannot increase a
conditional variance, so
$\var(\mu\mid y_i,\bt,\bo_c)\le\var(\mu\mid y_i,\bo_c)$ and
$\nu_i^\mu\le\nu_i^\perp$.
Since $x\mapsto x/(1-x)$ is increasing on $[0,1)$,
\eqref{eq:methods_kl_identity} gives the lower bound in
\eqref{eq:methods_kl_bound}.

For the upper bound, let $\bfSigma_o=\var(\bo_c)=\var(\bo_c\mid y_i)$ and
\[
\bfxi=\frac{\bfSigma_o^{-1/2}\bfgamma}{\var(\mu\mid y_i)^{1/2}},
\qquad
\bM=\bfSigma_o^{-1/2}\cov(\bo_c,\bt)\var(\bt)^{-1}
\cov(\bt,\bo_c)\bfSigma_o^{-1/2}.
\]
Then $\bP=\bfSigma_o^{1/2}(\bI-\bM)\bfSigma_o^{1/2}$, so $\bM$
is positive semidefinite with $\bI-\bM$ positive definite,
and Steps 4 and 5 give
$\nu_i^\mu=\|\bfxi\|^2$ and
$\nu_i^\perp=\bfxi^\top(\bI-\bM)^{-1}\bfxi$.
By its definition, $\nu_i^{\max}$ is the largest generalized
eigenvalue of $\cov(\bs_c,\bo_c\mid y_i)\bfSigma_o^{-1}
\cov(\bo_c,\bs_c\mid y_i)$ relative to $\var(\bs_c\mid y_i)$,
which is also the largest eigenvalue of
\[
\bfSigma_o^{-1/2}\cov(\bo_c,\bs_c\mid y_i)
\var(\bs_c\mid y_i)^{-1}\cov(\bs_c,\bo_c\mid y_i)
\bfSigma_o^{-1/2}.
\]
This matrix does not change when $\bs_c$ is replaced by the
invertible affine function $(\mu,\bt)$.
Given $y_i$, the covariance matrix of $(\mu,\bt)$ is block
diagonal with blocks $\var(\mu\mid y_i)$ and $\var(\bt)$, and
$\cov(\bo_c,(\mu,\bt)\mid y_i)=[\bfgamma,\cov(\bo_c,\bt)]$,
so the matrix equals $\bfxi\bfxi^\top+\bM$.
Moreover, $\bI-\bfxi\bfxi^\top-\bM
=\bfSigma_o^{-1/2}\var(\bo_c\mid y_i,\bs_c)\bfSigma_o^{-1/2}$
is positive definite, so
$\nu_i^{\max}=\lambda_{\max}(\bfxi\bfxi^\top+\bM)<1$.

If $\nu_i^\perp=0$, then $\nu_i^\mu=0$ by Step 5, and both
bounds hold with equality.
Otherwise, let $\bfphi=(\bI-\bM)^{-1}\bfxi$, so that
$\bM\bfphi=\bfphi-\bfxi$ and $\bfxi^\top\bfphi=\nu_i^\perp$.
The Rayleigh quotient at $\bfphi$ gives
\[
\nu_i^{\max}
\ge\frac{(\bfxi^\top\bfphi)^2+\bfphi^\top\bM\bfphi}{\|\bfphi\|^2}
=\frac{(\nu_i^\perp)^2+\|\bfphi\|^2-\nu_i^\perp}{\|\bfphi\|^2}
=1-\frac{\nu_i^\perp(1-\nu_i^\perp)}{\|\bfphi\|^2}.
\]
By the Cauchy-Schwarz inequality,
$\nu_i^\perp=\bfxi^\top\bfphi\le\|\bfxi\|\|\bfphi\|$, so
$\|\bfphi\|^2\ge(\nu_i^\perp)^2/\nu_i^\mu$ and
\[
1-\nu_i^{\max}
\le\frac{\nu_i^\perp(1-\nu_i^\perp)}{\|\bfphi\|^2}
\le\frac{\nu_i^\mu(1-\nu_i^\perp)}{\nu_i^\perp},
\]
which is equivalent to
$\nu_i^\perp/(1-\nu_i^\perp)\le\nu_i^\mu/(1-\nu_i^{\max})$.
Together with \eqref{eq:methods_kl_identity}, this gives the
upper bound in \eqref{eq:methods_kl_bound}.
\end{proof}

\section{Construction of Covariance Blocks}
\label{app:reduced_covariances}

The gradient covariance block omitted from
Section~\ref{sec:prelim_dgp} follows by differentiating the
kernel twice:
\begin{equation}
\begin{aligned}
\cov(\bg_j,\bg_{j'})
&=\nabla_{\bx_j}\nabla_{\bx_{j'}}^{\top}
k_{\bftheta}(\bx_j,\bx_{j'})
+\mathbbm{1}\{j=j'\}\sigma_g^2\bfLambda_{\bftheta}\\
&=-2\kappa_{\bftheta}'(r_{jj'})\bfLambda_{\bftheta}-4\kappa_{\bftheta}''(r_{jj'})\bfLambda_{\bftheta}\boldsymbol{\delta}_{jj'}
\boldsymbol{\delta}_{jj'}^\top\bfLambda_{\bftheta}
+\mathbbm{1}\{j=j'\}\sigma_g^2\bfLambda_{\bftheta}.
\end{aligned}
\label{eq:app_gg_covariance}
\end{equation}
In the scaled coordinates of Section~\ref{sec:methods_directions},
$\bfLambda_{\bftheta}=\bI_d$.

Fix a target input $\bx_i$ and a conditioning set
$c=\{j_1,\ldots,j_m\}$.
As in Section~\ref{sec:methods_directions}, we work
in scaled coordinates and omit the tildes.
Both retained directions lie in the column space of
$\bD_i$.
We use this representation to construct the reduced
covariance blocks directly from the Gram matrix
$\bH_i=\bD_i^\top\bD_i$, as in TERA.

Recall that
\[
\bD_i
=
[\bx_{j_1}-\bx_i,\ldots,\bx_{j_m}-\bx_i].
\]
Let $\be_a$ denote the $a$th coordinate vector
in $\bbR^m$.
Then
\[
\bx_i-\bx_{j_a}=-\bD_i\be_a,
\qquad
\bx_{j_b}-\bx_{j_a}=\bD_i(\be_b-\be_a).
\]
The squared distances in scaled coordinates are
\[
r_{j_a i}=[\bH_i]_{aa},
\qquad
r_{j_a j_b}
=
[\bH_i]_{aa}+[\bH_i]_{bb}-2[\bH_i]_{ab}.
\]
Thus the function covariance and target
cross-covariance are
\[
[\bK_{\by_c\by_c}]_{ab}
=
\kappa_{\bftheta}(r_{j_a j_b})
+\sigma_y^2\mathbbm{1}\{a=b\},
\qquad
[\bk_{\by_c y_i}]_a
=
\kappa_{\bftheta}(r_{j_a i}),
\]
and $\bw_i$ is obtained by solving
$\bK_{\by_c\by_c}\bw_i=\bk_{\by_c y_i}$.

For each $a=1,\ldots,m$, define
\begin{equation}
\bC_{i j_a}
=
\left[
-\eta_{j_a i}\be_a,\quad
\sum_{b=1}^m
[\bw_i]_b\,\eta_{j_a j_b}(\be_b-\be_a)
\right]
\in\bbR^{m\times2}.
\label{eq:app_direction_coefficients}
\end{equation}
By \eqref{eq:methods_conditional_covariance},
\begin{equation}
[\bu_{i j_a},\bv_{i j_a}]
=
\bD_i\bC_{i j_a},
\qquad
\bs_{i j_a}
=
\bC_{i j_a}^{\top}\bD_i^\top\bg_{j_a}.
\label{eq:app_coefficient_representation}
\end{equation}
The reduced observations can therefore be computed
from $\bD_i^\top[\bg_{j_1},\ldots,\bg_{j_m}]$
without explicitly forming the direction vectors.

To express the covariance blocks, write
\[
\bh_{ab}=\bH_i(\be_b-\be_a).
\]
Applying \eqref{eq:prelim_derivative_covariances} and
\eqref{eq:app_gg_covariance}
to \eqref{eq:app_coefficient_representation}, with
gradient noise covariance $\sigma_g^2\bI_d$ in
scaled coordinates, gives
\begin{equation}
\begin{aligned}
\cov(\bs_{i j_a},y_{j_b})
&=
\eta_{j_a j_b}\bC_{i j_a}^{\top}\bh_{ab} \in\bbR^2,\\
\cov(\bs_{i j_a},y_i)
&=
-\eta_{j_a i}\bC_{i j_a}^{\top}\bH_i\be_a \in\bbR^2,\\
\cov(\bs_{i j_a},\bs_{i j_b})
&=
\left[
\eta_{j_a j_b}
+\sigma_g^2\mathbbm{1}\{a=b\}
\right]
\bC_{i j_a}^{\top}\bH_i\bC_{i j_b}-
4\kappa_{\bftheta}''(r_{j_a j_b})
\bigl(\bC_{i j_a}^{\top}\bh_{ab}\bigr)
\bigl(\bC_{i j_b}^{\top}\bh_{ab}\bigr)^\top \in\bbR^{2\times2}.
\end{aligned}
\label{eq:app_reduced_gradient_blocks}
\end{equation}
Stacking these blocks gives
$\bK_{\bs_c\by_c}\in\bbR^{2m\times m}$,
$\bk_{\bs_c y_i}\in\bbR^{2m}$, and
$\bK_{\bs_c\bs_c}\in\bbR^{2m\times2m}$.

\paragraph{Batched training and prediction.}
For a batch $\mathcal B$ of $B$ prediction targets, we stack
the reduced covariance matrices of
Section~\ref{sec:methods_learning_prediction} as
\[
\mathcal K
=
\operatorname{stack}_{\star\in\mathcal B}(\bK_\star)
\in\bbR^{B\times3m\times3m}.
\]
Batched Cholesky factorizations and triangular solves
compute
\begin{equation}
\begin{aligned}
\bL_\star\bL_\star^\top&=\bK_\star,\\
\bL_\star
\begin{bmatrix}
\bfalpha_\star&\bfbeta_\star
\end{bmatrix}
&=
\begin{bmatrix}
\bz_\star^R-\bbE[\bz_\star^R]&\bk_\star
\end{bmatrix},
\qquad \star\in\mathcal B.
\end{aligned}
\label{eq:app_batched_solves}
\end{equation}
By \eqref{eq:prelim_posterior_updates}, the predictive means
and variances then follow from
\begin{equation}
\begin{aligned}
\mu_\star^R
&=\bbE[f_\star]+\bfalpha_\star^\top\bfbeta_\star,\\
V_\star^R
&=\var(f_\star)-\bfbeta_\star^\top\bfbeta_\star.
\end{aligned}
\label{eq:app_batched_prediction}
\end{equation}

\section{Experimental Settings}
\label{app:experiment_settings}

\subsection{Scaling Conditioning Sets and Batch Sizes}
\label{app:scaling_settings}

We evaluate conditioning set scaling under an ARD
Mat\'ern-$5/2$ GP and $m$ scaling with $B=32$ and $B=1{,}024$ on
the MD22 Buckyball-catcher dataset. The GP experiment examines how predictive accuracy changes
with the conditioning set size when covariance parameters
are fixed. The MD22 experiment separately measures how
conditioning set and batch sizes affect the computational
cost of training. 
Table~\ref{tab:scaling_settings} summarizes the settings.
All methods use identical inputs, conditioning sets,
and batch sizes.
For the MD22 experiment, we use the initial covariance parameters in
Table~\ref{tab:md22_settings}.

\begin{table}[h]
\centering
\small
\setlength{\tabcolsep}{5pt}
\renewcommand{\arraystretch}{1.12}
\caption{Settings for conditioning set and batch size scaling.}
\label{tab:scaling_settings}
\begin{tabular}{ll}
\toprule
Setting & Value \\
\midrule
\multicolumn{2}{l}{GP prediction} \\
\midrule
Training/test inputs & $100,000$ / $256$ \\
Input dimension & $d=500$ \\
Input design & Sobol in $[0,1]^d$ \\
Kernel & ARD Mat\'ern-$5/2$ \\
Signal variance & $1$ \\
Lengthscales & Logarithmically spaced; $5.82$--$23.27$ \\
Function-noise variance & $4$ \\
Gradient-noise variance & $0.36$ per coordinate \\
Conditioning set sizes & $20,40,80,160,320$ \\
Prediction batch size & $32$ \\
Repeats & $5$ \\
Covariance parameters & Fixed at their true values \\
\midrule
\multicolumn{2}{l}{MD22 training step} \\
\midrule
Dataset & Buckyball-catcher \\
Train/test split & 90/10 \\
Repeats & $5$ \\
Kernel & Isotropic SE \\
Conditioning set sizes & $10,20,30,40,50,60,80,100,150,200$ \\
Batch sizes & $32$, $1,024$ \\
Measured computation & One forward and backward pass \\
Covariance parameters & Fixed at their initial values \\
\bottomrule
\end{tabular}
\end{table}

Dense joint sampling is infeasible at this scale.
Instead, we evaluate prediction error analytically under
the assumed GP.
For each target $f_\star$, let $\bz_\star$ collect the
observations retained by the method, with
\[
\bK_\star=\cov(\bz_\star),
\qquad
\bk_\star=\cov(\bz_\star,f_\star),
\qquad
k_{\star\star}=\var(f_\star).
\]
The observation covariance $\bK_\star$ includes both
function and gradient noise.
Under the zero-mean GP, the predictor is
$\hat f_\star=\bfalpha_\star^\top\bz_\star$, where
$\bfalpha_\star=\bK_\star^{-1}\bk_\star$.
With inputs and covariance parameters fixed,
\begin{align*}
\bbE\!\left[(f_\star-\hat f_\star)^2\right]
&=
\bbE[f_\star^2]
-2\bfalpha_\star^\top\bbE[\bz_\star f_\star]
+\bfalpha_\star^\top
\bbE[\bz_\star\bz_\star^\top]\bfalpha_\star\\
&=
k_{\star\star}
-2\bfalpha_\star^\top\bk_\star
+\bfalpha_\star^\top\bK_\star\bfalpha_\star.
\end{align*}
Thus, we have
\[
\mathrm{RMSE}
=
\left[
\frac{1}{n_\star}\sum_{\star=1}^{n_\star}
\left(
k_{\star\star}
-2\bfalpha_\star^\top\bk_\star
+\bfalpha_\star^\top\bK_\star\bfalpha_\star
\right)
\right]^{1/2}.
\]
GP runtime covers prediction over all $256$ targets,
including covariance assembly, Cholesky factorization,
and triangular solves.
MD22 runtime covers one forward and backward pass,
excluding preprocessing, conditioning-set construction,
and optimizer updates.

\subsection{MD22 Regression}

We use six MD22 benchmarks, DHA ($n=69,753$, $d=168$),
AT-AT ($n=20,001$, $d=180$), Stachyose ($n=27,272$, $d=261$),
AT-AT-CG-CG ($n=10,153$, $d=354$), Buckyball-catcher ($n=6102$, $d=444$), and Double-walled nanotube ($n=5032$, $d=1110$). Table~\ref{tab:md22_settings} summarizes the experimental settings adapted from~\cite{huang2026scaling,Seung2026ScalableDG}. For \lite, we use the same settings as TERA, including conditioning set size, learning rate, parameter update step, and batch size. We use $m=30$ and batches of $32$ targets to compare the
methods in a regime where TERA's exact reduction remains
computationally affordable.
This comparison evaluates the accuracy and cost of \lite's
two-direction approximation at a common conditioning set
size. Its ability to accommodate larger conditioning sets
and batches is evaluated separately in the scaling experiments.

\begin{table}[h]
\centering
\small
\setlength{\tabcolsep}{5pt}
\renewcommand{\arraystretch}{1.12}
\caption{Settings for MD22 regression.}
\label{tab:md22_settings}
\begin{tabular}{ll}
\toprule
Setting & Value \\
\midrule
Train/test split & 90/10 \\
Seeds & 1,27,42,86,99 \\
Kernel & isotropic SE \\
Initial lengthscale & 1.0 \\
Initial signal variance & 1.0 \\
Initial value-noise variance & $10^{-3}$ \\
Initial gradient-noise variance
& $10^{-3}d$ for DSoftKI, $10^{-3}$ otherwise \\
Conditioning set size & $m=30$ \\
Optimizer & Adam; learning rates below \\
\midrule
Standard GP & 50 updates; lr $0.01$, or $0.005$ for nanotube \\
TERA & 1 epoch; lr $0.01$, batch size $32$\\
\lite & 1 epoch; lr $0.01$, batch size $32$ \\
DDSVGP & 50 epochs, 512 inducing inputs; 2 directions \\
DSoftKI & 50 epochs, 512 interpolation points, float32 \\
\midrule
\multicolumn{2}{l}{DDSVGP and DSoftKI training schedule} \\
\midrule
$d<180$ & batch 1024; lr 0.012 / 0.008 \\
$180\le d<300$ & batch 512; lr 0.006 / 0.004 \\
$300\le d<1000$ & batch 256; lr 0.003 / 0.002 \\
$d\ge1000$ & batch 128; lr 0.0015 / 0.001 \\
\bottomrule
\end{tabular}
\end{table}

\clearpage
\subsection{Bayesian Optimization}

We evaluate Ackley, Levy, and LassoDNA using the settings
in Table~\ref{tab:bo_settings}, adapted from
\cite{Seung2026ScalableDG}.
The synthetic objectives test optimization performance
in hundreds of dimensions, while LassoDNA evaluates
gradient-informed BO on a real-world hyperparameter
optimization task.
The function-only methods use objective values, whereas
TERA and \lite additionally use gradients at the evaluated
inputs.
All methods share the initial Sobol design within each run,
and objective values are standardized before surrogate fitting.
TuRBO-1 uses LogEI~\citep{Ament2023UnexpectedIT} on
the synthetic objectives and Thompson
sampling~\citep{Thompson1933ONTL} on LassoDNA.
TERA and \lite use the same conditioning set size and
acquisition optimization settings.
Wall-clock time includes surrogate fitting, acquisition
optimization, and objective evaluation, including
hypergradient computation for TERA and \lite.

\begin{table*}[h]
\centering
\small
\setlength{\tabcolsep}{6pt}
\renewcommand{\arraystretch}{1.15}
\caption{Settings for BO synthetic objectives and LassoDNA benchmark.}
\label{tab:bo_settings}

\begin{tabularx}{\textwidth}{
@{}l>{\raggedright\arraybackslash}X@{}}
\toprule
Benchmarks
& Ackley ($d=500$), Levy ($d=800$),
  LassoDNA ($d=180$) \\
Domain & $[0,1]^d$ \\
Iterations / initial design
& $100$ iterations / $30$ Sobol points \\
Seeds & $1,\ 27,\ 42,\ 86,\ 99$ \\
Noise parameters
& $10^{-3}$, fixed on synthetic tasks;
  initialized at $10^{-3}$ and learned on LassoDNA \\
\bottomrule
\end{tabularx}

\medskip

\begin{tabularx}{\textwidth}{
@{}>{\raggedright\arraybackslash}p{0.26\textwidth}
*{4}{>{\centering\arraybackslash}X}@{}}
\toprule
& VBO & TuRBO-1 & TERA & \lite \\
\midrule
Kernel
& SE & Mat\'ern-$5/2$ & Mat\'ern-$5/2$ & Mat\'ern-$5/2$ \\
Conditioning-set size
& -- & -- & $30$ & $30$ \\
Parameter estimation
& MAP & Adam & Adam & Adam \\
Lengthscale prior
& Dimension-scaled & -- & -- & -- \\
\addlinespace
Initial Adam updates
& -- & $50$ & $50$ & $50$ \\
Adam updates per refit
& -- & $50$ & $10$ & $10$ \\
Refit interval
& Every iteration & Every iteration
& $20$ iterations & $20$ iterations \\
\midrule
Acquisition on synthetic tasks
& LogEI & LogEI & LogEI & LogEI \\
Acquisition on LassoDNA
& LogEI & TS & LogEI & LogEI \\
Raw candidates
& $512$ & $512$ & $256$ & $256$ \\
Restarts
& $4$ & $4$ & $4$ & $4$ \\
Maximum L-BFGS iterations
& $300$ & $300$ & $30$ & $30$ \\
\bottomrule
\end{tabularx}
\end{table*}

LassoDNA is a hyperparameter optimization task in which
BO selects $180$ regularization weights for a weighted
Lasso model, searching on a logarithmic scale.
The objective is the mean validation MSE across
cross-validation folds.
TERA and \lite compute derivatives of the validation
loss with respect to these weights by implicitly
differentiating the fitted Lasso
optimality condition~\citep{Bertrand2020ImplicitDO}.
We apply the chain rule to express these derivatives
in the BO coordinates and average them across folds.

\clearpage
\section{Additional Results}
\label{app:additional_results}

\subsection{Examining the Accuracy Gap on MD22}
\label{app:md22_shared_parameters}

We examine whether the accuracy gap between \lite and
TERA on MD22 is removed by sharing covariance parameters,
increasing the conditioning set size, or further training.

\paragraph{Prediction with shared covariance parameters.}
For each dataset and seed, we train TERA at $m=20$
for one epoch and freeze its learned covariance parameters.
Using these parameters for all methods, we evaluate
\lite and a function-only Vecchia GP at
$m\in\{20,40,60,80,160,320\}$, and TERA at
$m\in\{20,40,60\}$.
The function-only Vecchia GP uses the same conditioning
inputs but retains only the energy (function value) observations. This baseline is distinct from the Standard GP used
in the main MD22 comparison.

Figure~\ref{fig:appendix_md22_fixed_parameters} shows that
\lite's RMSE remains higher than TERA's at matched
conditioning set sizes, even when both methods use the
same covariance parameters.
Increasing $m$ provides little improvement for \lite
or the function-only Vecchia GP.
Under these TERA-estimated parameters, \lite has lower
RMSE than the function-only Vecchia GP on AT-AT and
Buckyball-catcher, but higher RMSE on the other four datasets.
Thus, using common covariance parameters and larger
conditioning sets does not remove the accuracy gap in
these comparisons.

\begin{figure}[H]
\centering
\includegraphics[width=0.99\linewidth]{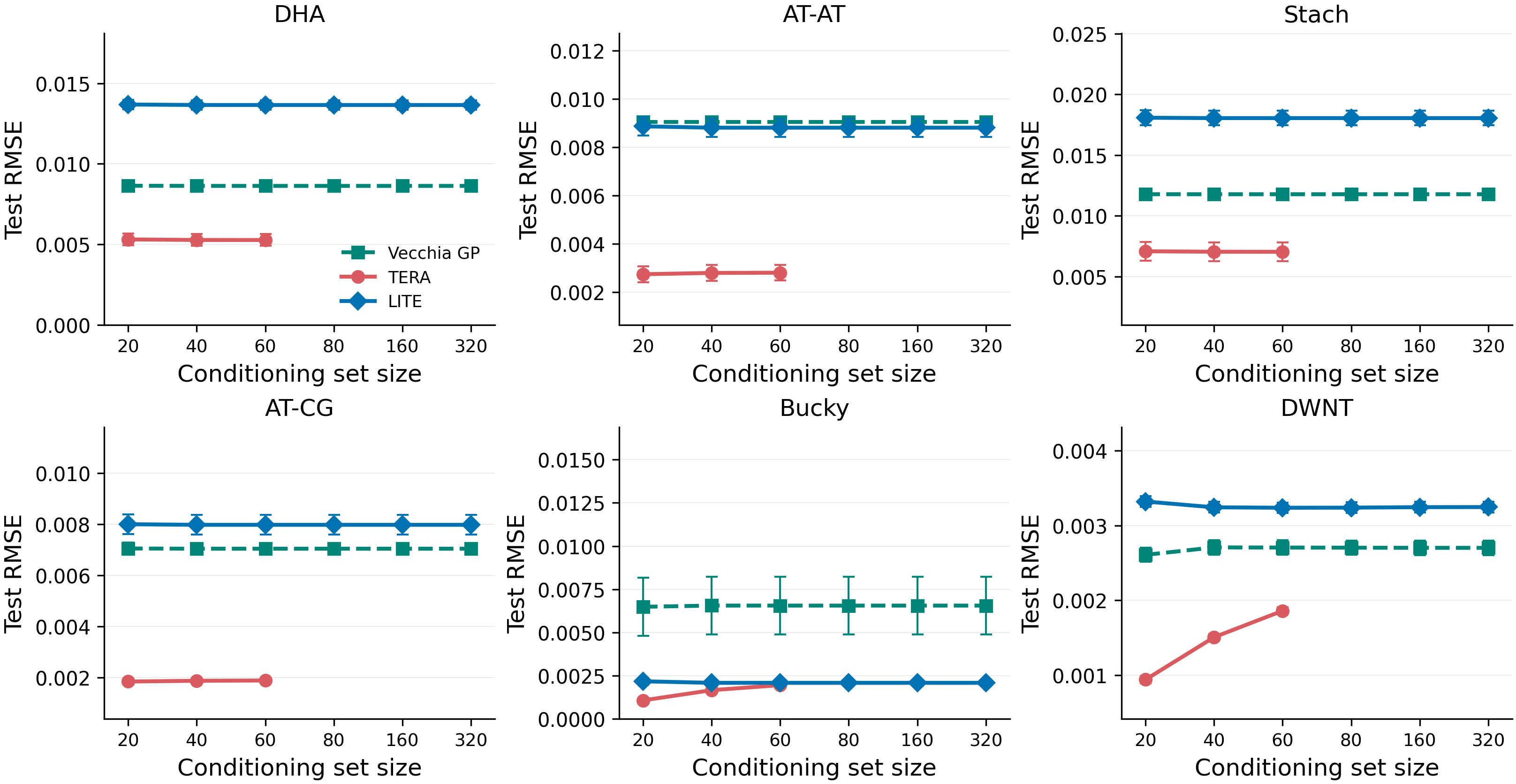}
\caption{Test RMSE versus conditioning set size on six
MD22 datasets, with all methods using the covariance
parameters learned by TERA at $m=20$ after one epoch.
Parameters remain fixed throughout prediction.
Increasing $m$ provides little improvement for \lite
and the function-only Vecchia GP, and \lite's RMSE remains
higher than TERA's at matched $m$.}
\label{fig:appendix_md22_fixed_parameters}
\end{figure}

\clearpage
\paragraph{Further training from shared parameters.}
We next initialize \lite with the same TERA estimates
and independently train it for one additional epoch
at each conditioning set size.
TERA's parameters remain fixed.
Figure~\ref{fig:appendix_md22_warmstart} compares these
results with prediction using the frozen parameters.

Further training substantially reduces \lite's RMSE
on DHA, AT-AT, Stachyose, and Double-walled nanotube,
but changes little on AT-AT-CG-CG and increases RMSE
on Buckyball-catcher.
The effect of further parameter estimation is therefore
dataset dependent, and increasing $m$ again provides
little additional improvement.

\begin{figure}[H]
\centering
\includegraphics[width=0.99\linewidth]{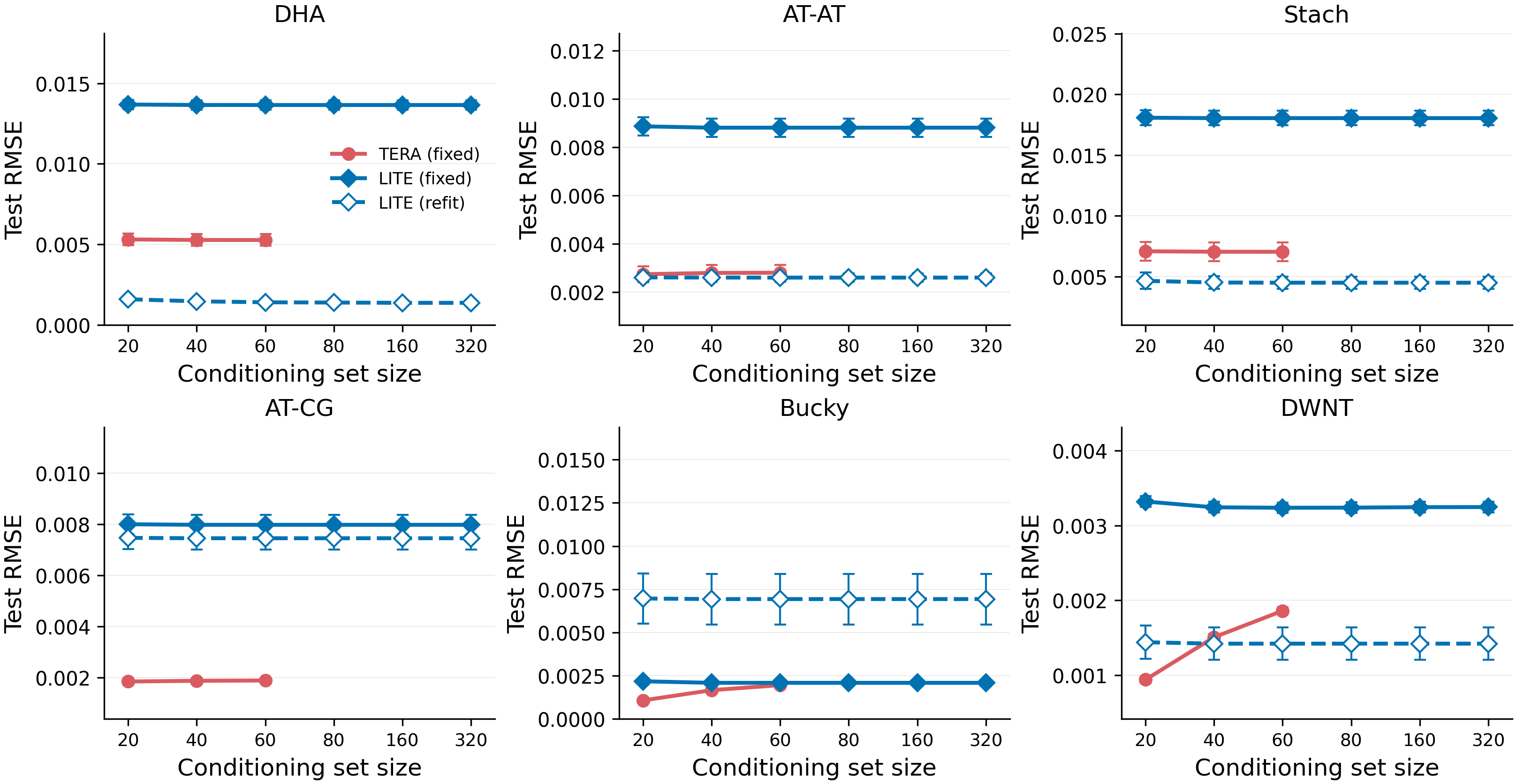}
\caption{Test RMSE versus conditioning set size on six
MD22 datasets, using the same TERA estimates as in
Figure~\ref{fig:appendix_md22_fixed_parameters}.
Solid curves keep these parameters fixed; the dashed
curve trains \lite for one additional epoch at each $m$,
starting independently from the same estimates.
Further training improves \lite on four datasets, changes
little on AT-AT-CG-CG, and worsens prediction on
Buckyball-catcher.}
\label{fig:appendix_md22_warmstart}
\end{figure}

\clearpage
\paragraph{Learning curves from a common initialization.}
\label{app:md22_learning_curves}

In a separate experiment, we train TERA and \lite from
the same initial covariance parameters for five epochs
at $m=20$, using training and prediction batches of
$32$ targets.
This experiment examines how additional training affects
test error without initializing either method from
the other's parameter estimates.

Figure~\ref{fig:appendix_md22_training_curve} shows that
test RMSE does not decrease monotonically with training.
On several datasets, it increases after an early minimum,
and later improvements vary across datasets.
These learning curves complement the shared-parameter
comparisons by showing that additional training does not
uniformly improve predictive accuracy.

\begin{figure}[H]
\centering
\includegraphics[width=0.99\linewidth]{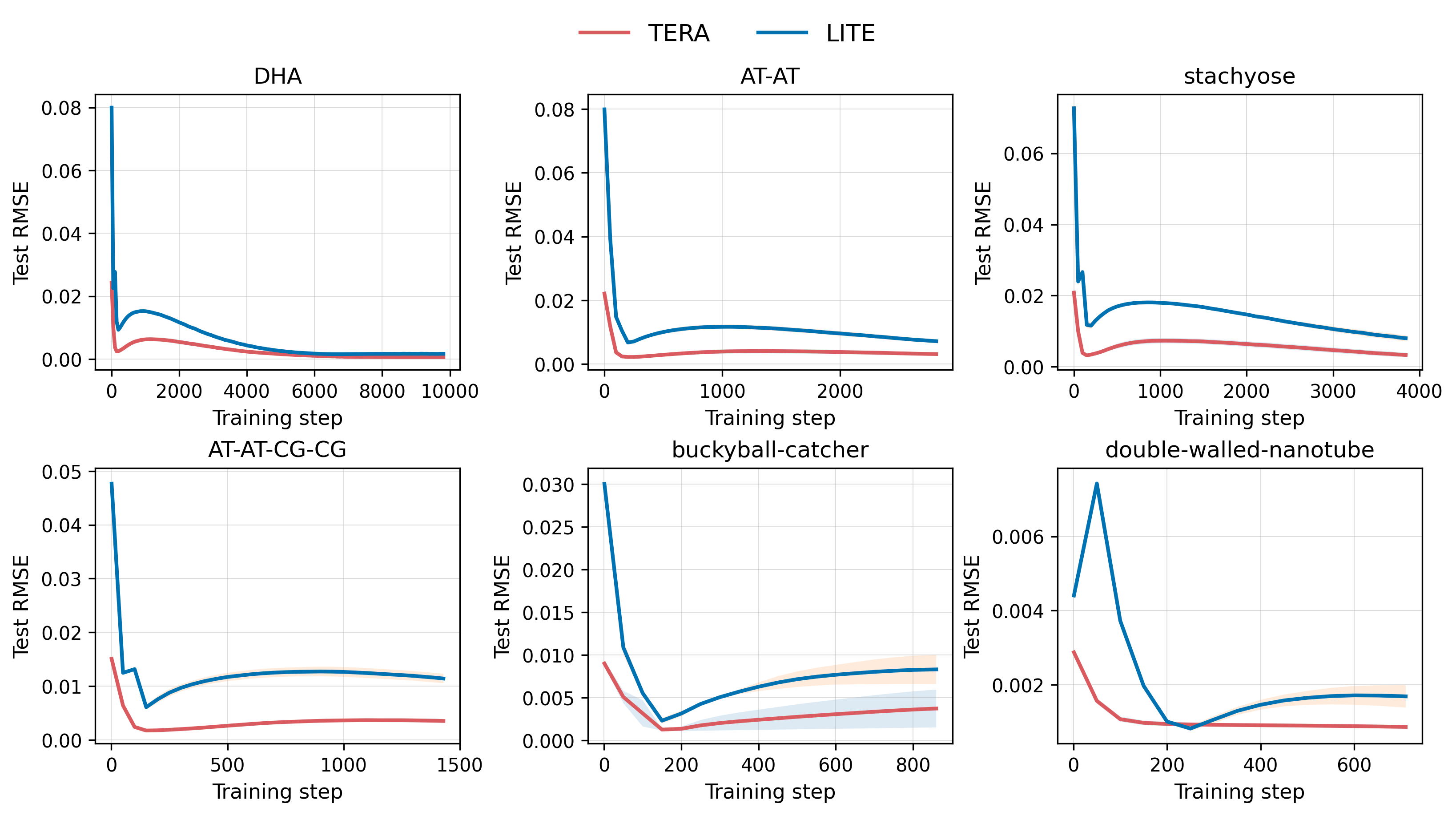}
\caption{Test RMSE versus training step on six MD22
datasets. Both methods use $m=20$, training and prediction
batch sizes of $32$, and five training epochs.
Additional updates do not uniformly improve predictive
accuracy, with RMSE increasing after early minima on some
datasets. Curves average five independent runs.}
\label{fig:appendix_md22_training_curve}
\end{figure}

\paragraph{Summary of the diagnostics.}
These diagnostics show that common covariance parameters
do not by themselves remove the accuracy gap, while the
effect of further training varies across datasets.
On MD22, all three Vecchia methods show little
improvement from larger conditioning sets over the
tested ranges.
The main MD22 comparison demonstrates an accuracy-cost
tradeoff, with \lite achieving substantial runtime and
memory savings at the cost of higher RMSE than TERA.
The GP simulation in Section~\ref{sec:experiments_scaling}
demonstrates the predictive benefit of larger conditioning
sets.
In that setting, \lite uses sets beyond TERA's memory
limit to achieve lower error at substantially lower
computational and memory cost.

\clearpage
\subsection{Computational Cost of Bayesian Optimization}
\label{app:bo_computational_cost}

Figure~\ref{fig:appendix_bo_resources} complements the
optimization curves in Figure~\ref{fig:experiment_bo}
by comparing wall-clock time and peak GPU memory
for the same BO runs. Across all three tasks, \lite uses only approximately
$13$--$23\%$ of TERA's peak GPU memory.
Its resource advantages also extend to function-only baselines.
On Levy, \lite runs $2.7\times$ faster than VBO and uses only
$35\%$ of TuRBO-1's peak GPU memory.
On LassoDNA, it uses only $12\%$ of TuRBO-1's memory and
requires less wall-clock time than VBO, including the cost
of hypergradient computation.
These savings accompany the effective optimization shown in the main text.

\begin{figure}[H]
\centering
\includegraphics[width=0.99\linewidth]{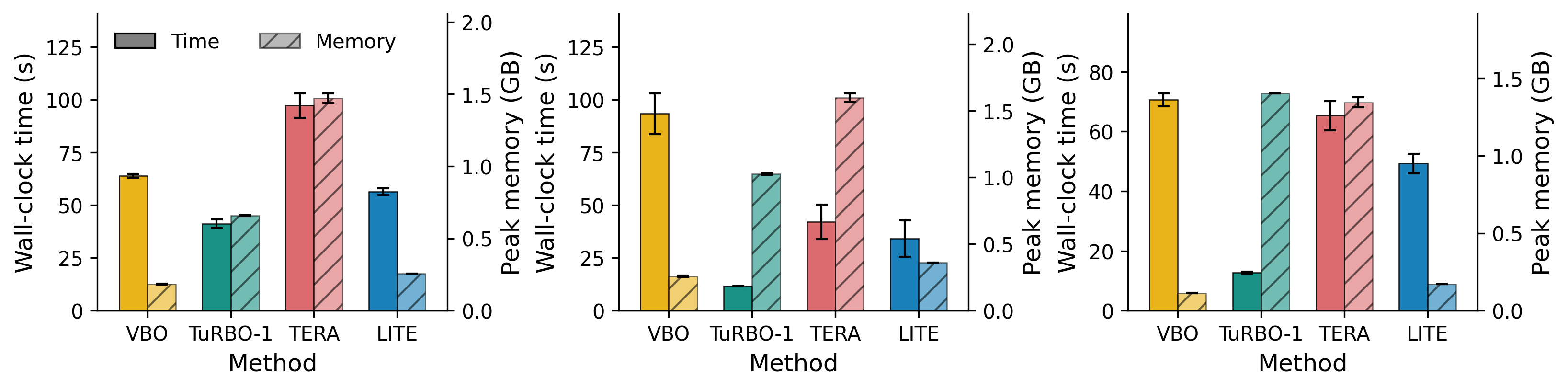}
\caption{Wall-clock time and peak GPU memory on Ackley
($500D$, left), Levy ($800D$, center), and LassoDNA
($180D$, right). Time includes objective evaluations and gradient computation where used, including hypergradient computation on LassoDNA. Bars average five independent runs, with error bars indicating
standard errors.}
\label{fig:appendix_bo_resources}
\end{figure}

\end{document}